\documentclass{article}

  \usepackage[preprint]{neurips_2026}

\usepackage{comment}
\usepackage{tabularx}
\usepackage[utf8]{inputenc} 
\usepackage[T1]{fontenc}    
\usepackage{hyperref}
\usepackage{url}            
\usepackage{booktabs}       
\usepackage{amsfonts}       
\usepackage{nicefrac}       
\usepackage{microtype}      
\usepackage{xcolor}         
\usepackage{amsmath}        
\usepackage{amssymb}        
\usepackage{mathtools}      
\usepackage{amsthm}         
\usepackage{bm}
\usepackage{cleveref}
\usepackage{tikz}
\DeclareRobustCommand{\E}{\mathop{\mathbb{E}}\nolimits}
\usetikzlibrary{arrows.meta}
\usepackage{thmtools}
\usepackage{thm-restate}
\newtheorem{theorem}{Theorem}
\newtheorem{proposition}{Proposition}
\newtheorem{lemma}{Lemma}
\newtheorem{corollary}{Corollary}
\theoremstyle{definition}
\newtheorem{definition}{Definition}
\newtheorem{assumption}{Assumption}

\theoremstyle{remark}
\newtheorem{remark}{Remark}
\usepackage{enumitem}
\usepackage{placeins}
\newcommand{\R}{\mathbb{R}}

\DeclareMathOperator*{\essinf}{ess\,inf}
\DeclareMathOperator*{\esssup}{ess\,sup}

\newcommand{\Wone}{\mathsf W_1}
\newcommand{\Wtwo}{\mathsf W_2}

\newcolumntype{Y}{>{\raggedright\arraybackslash}X}
\newcolumntype{P}[1]{>{\raggedright\arraybackslash}p{#1}}
\title{Load--Reserve Wasserstein Propagation for Isotropic Diffusion Samplers}

\author{%
  Zicheng Lyu \\
  School of Data Science, Fudan University \\
  \texttt{lyuzicheng@gmail.com} \\
  \And
  Zengfeng Huang\thanks{Corresponding author.} \\
  School of Data Science, Fudan University \\
  Shanghai Innovation Institute \\
  \texttt{huangzf@fudan.edu.cn} \\
}

\begin{document}

\maketitle

\begin{abstract}
Many Wasserstein analyses of diffusion samplers control reverse-time propagation
by global stability summaries of the learned drift. These summaries can hide
radial geometry: equal-height expansive regions of different width can yield
different propagation costs. We give a profile-adapted propagation interface for
scalar-isotropic reverse-SDE windows with certified learned-drift profiles. A
certified lower radial profile is compiled into an affine-tail transportation cost:
reflection coupling reduces stability to a one-dimensional slope budget, and
Hardy capacity quantifies the load paid before a contractive tail reserve. The
compiler yields an adapted cost, contraction rate, and retained tail slope.
Score-modeling and solver residuals are treated as forcing inputs and propagate
additively in the adapted Wasserstein distance. Quadratic Wasserstein error is
reported only at terminal time, using the retained tail slope with tail, moment,
or support information. Gaussian-smoothed denoising geometry supplies
inverse-radius profiles for uniformly dissipative, bounded-amplitude, and
common-covariance mixture windows. Fixed-height examples show that adverse
height, even with eventual reserve, does not determine the certificate; barrier
examples show that the load dependence is structural.
\end{abstract}
\section{Introduction}
\label{sec:introduction}

Score-based diffusion models are now a standard framework for generative modeling
\citep{sohl_dickstein2015nonequilibrium,song2019score,ho2020ddpm,song2021sde}.
Their nonasymptotic analyses usually decompose error into score estimation,
reverse-time propagation, discretization, and terminal conversion
\citep{chen2022sampling,lee2023general,chen2023improved,chen2023pfode_fast,
gao2024pfode,benton2024nearly,gao2025general_w2,wang2024gaussian_tail,
li2025odt,yu2025discretization,chen2026high_accuracy,pfarr2026higher_order}.
We study the continuous propagation module for scalar-isotropic reverse SDEs:
treating score and solver error envelopes as forcing inputs, we ask which certified feature of the
learned reverse dynamics controls whether mismatch is amplified or dissipated.

A widely used way to close this propagation step is to summarize the learned
drift by a global stability scale, such as a Lipschitz, Hessian, semiconvexity,
or dissipativity constant
\citep{chen2022sampling,chen2023improved,lee2023general,gao2025general_w2,bruno2025semiconvex}.
These summaries are natural and often effective when small and verifiable. Our contribution is a load--reserve radial stability certificate: the learned drift may be
locally expansive if the expansive radial region has finite width-weighted load before a
contractive tail reserve.

The certificate is compiled into the metric used for propagation. Reflection
coupling reduces learned-flow stability to a one-dimensional slope budget: a
concave transportation cost spends derivative while crossing the adverse core and
then retains an affine tail. The theorem propagates score-modeling and solver
residuals in this adapted cost, rather than requiring a direct \(\Wtwo^2\)
propagation estimate. Quadratic Wasserstein error is reported only at terminal
time, using the retained affine-tail slope together with tail, moment, or support
information.

Gaussian smoothing supplies structured windows where such certificates can be
computed. The exact score is a Gaussian inward pull plus a smoothed denoiser; when
the pull dominates pairwise denoiser expansion, the learned drift has an
inverse-radius load--reserve profile. This covers uniformly dissipative windows,
bounded-amplitude transfers, and common-covariance Gaussian-mixture smoothing
windows. Thus analytic or verified learned-flow profiles can be translated into
propagation rates and terminal reporting slopes.

\begin{figure*}[t]
\centering
\begin{tikzpicture}[
    >=Latex,
    font=\footnotesize,
    line cap=round,
    line join=round,
    every node/.style={inner sep=1.0pt},
    panel/.style={draw=gray!55, rounded corners=2.2pt, fill=gray!2},
    axis/.style={gray!65, ->},
    title/.style={font=\bfseries\footnotesize},
    soft/.style={gray!75}
]

\def\Wtop{6.00}
\def\Htop{2.85}
\def\Wbot{8.40}
\def\Hbot{3.00}

\begin{scope}[shift={(0,3.55)}]
  \draw[panel] (0,0) rectangle (\Wtop,\Htop);
  \node[anchor=west,title] at (0.16,\Htop-0.18)
    {A. height loses width};

  \draw[axis] (0.50,0.58) -- (5.60,0.58);
  \draw[axis] (0.60,0.42) -- (0.60,2.35);
  \node[anchor=south west,soft] at (0.74,2.14) {$\bar\kappa(r)$};
  \node[anchor=north east,soft] at (5.55,0.50) {$r$};

  \draw[gray!55,dashed] (0.82,1.72) -- (5.35,1.72);
  \draw[gray!45,dashed] (0.82,0.88) -- (5.35,0.88);

  \fill[orange!25]
    (1.05,1.72)
      .. controls (1.15,1.38) and (1.18,1.05) .. (1.30,0.88)
      -- (1.80,0.88)
      .. controls (1.92,1.05) and (1.95,1.38) .. (2.05,1.72)
      -- cycle;
  \draw[orange!85!black,thick]
    (0.90,1.72)
      .. controls (1.10,1.72) and (1.15,1.05) .. (1.30,0.88)
      -- (1.80,0.88)
      .. controls (1.95,1.05) and (2.00,1.72) .. (2.20,1.72);

  \fill[purple!18]
    (2.80,1.72)
      .. controls (2.95,1.25) and (3.05,0.95) .. (3.25,0.88)
      -- (4.95,0.88)
      .. controls (5.15,0.95) and (5.25,1.25) .. (5.40,1.72)
      -- cycle;
  \draw[purple!75!black,thick]
    (2.60,1.72)
      .. controls (2.95,1.72) and (3.05,0.95) .. (3.25,0.88)
      -- (4.95,0.88)
      .. controls (5.15,0.95) and (5.25,1.72) .. (5.50,1.72);

  \draw[<->,gray!70] (2.32,1.72) -- (2.32,0.88);
  \node[gray!70,anchor=west] at (2.38,1.22) {$B$};
  \node[orange!85!black] at (1.52,0.66) {small $A_R$};
  \node[purple!75!black] at (4.10,0.66) {large $A_R$};
  \node[soft] at (3.48,2.22) {same height, different load};
\end{scope}

\begin{scope}[shift={(6.55,3.55)}]
  \draw[panel] (0,0) rectangle (\Wtop,\Htop);
  \node[anchor=west,title] at (0.16,\Htop-0.18)
    {B. reflection reduces to the radius};

  \fill[blue!75!black] (1.15,0.92) circle (1.8pt);
  \fill[red!75!black]  (4.90,1.98) circle (1.8pt);
  \node[blue!75!black,below left] at (1.15,0.92) {$x$};
  \node[red!75!black,above right] at (4.90,1.98) {$y$};

  \draw[black!80,thick]
    (1.15,0.92) -- node[midway, above=4pt, sloped, fill=white, inner sep=1.2pt]
    {$r=\|x-y\|$} (4.90,1.98);

  \draw[->,black!70] (1.72,1.08) -- (0.98,1.36);
  \draw[->,black!70] (4.34,1.82) -- (5.08,1.54);

  \node[soft] at (3.02,2.34) {$e=(x-y)/\|x-y\|$};
  \node[align=center] at (3.02,0.76)
    {$\langle \widehat b_t(x)-\widehat b_t(y),\, e\rangle$};
\end{scope}

\begin{scope}[shift={(2.05,0)}]
  \draw[panel] (0,0) rectangle (\Wbot,\Hbot);
  \node[anchor=west,title] at (0.16,\Hbot-0.18)
    {C. certified radial profile: load before reserve};

  \draw[axis] (0.55,1.48) -- (7.85,1.48);
  \draw[axis] (0.68,0.42) -- (0.68,2.50);
  \node[anchor=south west,soft] at (0.80,2.15) {$\bar\kappa(r)$};
  \node[anchor=north east,soft] at (7.78,1.40) {$r$};

  \draw[gray!55] (4.90,0.62) -- (4.90,2.35);
  \node[fill=white] at (4.90,0.84) {$R$};

  \draw[red!70!black,dashed,thick] (5.05,2.00) -- (7.55,2.00);
  \node[red!70!black,anchor=west] at (5.58,2.16) {$m_R$};

  \fill[orange!28]
    (1.60,1.48)
      .. controls (1.72,1.34) and (1.78,1.28) .. (1.82,1.24)
      .. controls (1.96,1.06) and (2.12,0.86) .. (2.28,0.76)
      .. controls (2.42,0.70) and (2.58,0.79) .. (2.70,0.88)
      .. controls (2.84,0.95) and (2.98,1.01) .. (3.08,1.06)
      .. controls (3.20,1.12) and (3.36,1.20) .. (3.52,1.28)
      .. controls (3.62,1.34) and (3.72,1.41) .. (3.79,1.48)
    -- cycle;

  \draw[teal!70!black,thick,smooth]
    plot coordinates {
      (0.95,1.88) (1.38,1.72) (1.82,1.24) (2.28,0.76) (2.70,0.88)
      (3.08,1.06) (3.52,1.28) (3.95,1.60) (4.42,1.90)
      (4.90,1.97) (5.80,1.99) (7.20,2.00)
    };

  \node[orange!85!black] at (2.50,1.10) {$A_R$};
  \node[soft] at (2.68,0.58) {adverse core};
  \node[soft] at (6.30,0.82) {contractive tail};
\end{scope}

\draw[gray!55,->,thick] (6.00,3.10) -- (6.00,2.98);
\draw[gray!55,->,thick] (7.10,3.10) -- (7.10,2.98);

\end{tikzpicture}
\caption{Core intuition. (\textbf{A}) A scalar adverse-height summary can miss
how much radial width a concave metric must cross. (\textbf{B}) Under same-drift
isotropic reflection coupling of learned copies, an increasing radial cost sees
the learned drift through its projection onto the separation direction.
(\textbf{C}) The certified radial profile is compressed into adverse load
\(A_R\) before a cutoff and contractive reserve \(m_R\) beyond it; the affine
tail retains slope for terminal \(\Wtwo^2\) reporting.}
\vspace{-0.4em}\label{fig:intro-intuition}
\end{figure*}

Figure~\ref{fig:intro-intuition} summarizes the mechanism. The following
window-level statement is the main propagation bound in informal form.

\begin{theorem}[Informal]
On one certified scalar-isotropic window \([0,\tau]\), let \(\mu_t\) be the
ideal law and \(\widetilde\mu_t\) the implemented law. Assume
\(\sigma_t\in[\sigma_-,\sigma_+]\), \(\sigma_->0\), and let the learned drift have
profile \(\bar\kappa\). For a cutoff \(R>0\), set
\[
A_R:=\int_0^R[-\bar\kappa(r)]_+r\,dr,
\qquad
m_R:=\essinf_{r\ge R}\bar\kappa(r).
\]
If \(A_R<\infty\) and \(m_R>0\), then there is an explicit increasing concave
cost \(\phi_R\), affine on \([R,\infty)\), such that
\[
\mathsf W_{\phi_R}(\widetilde\mu_\tau,\mu_\tau)
\lesssim
 e^{-\rho_R\tau}\Wtwo(\widetilde\mu_0,\mu_0)
+
\int_0^\tau e^{-\rho_R(\tau-t)}(\eta_t+\zeta_t)\,dt,
\]
where \(\eta_t\) is the score-modeling residual, \(\zeta_t\) is the solver
residual, and
\[
\rho_R\gtrsim
\left(m_R\wedge\frac{\sigma_-^2}{R^2}\right)
\exp\!\left\{-\frac{A_R}{2\sigma_-^2}\right\}.
\]
A terminal reporting step converts the affine-tail bound into \(\Wtwo^2\), using
the tail slope of \(\phi_R\) and terminal tail, moment, or support data. See
\Cref{thm:affine-tail-metric} and \Cref{thm:main-certificate},
\Cref{cor:terminal-w2-reporting}, and \Cref{prop:window-composition}.
\end{theorem}

\paragraph{Contributions.}
\begin{itemize}[leftmargin=1.25em,itemsep=1pt,topsep=2pt,parsep=0pt,partopsep=0pt]
    \item \emph{Forcing--stability--reporting interface.}
    We separate the inputs in reverse-SDE error propagation:
score-modeling and solver residuals are forcing terms, a certified pairwise radial
profile is the learned-flow stability input, and \(\Wtwo^2\) is recovered only by terminal
reporting.

    \item \emph{Load--reserve affine-tail compiler.}
    Reflection radialization, Hardy capacity, and
an explicit slope budget compile \((A_R,m_R,R)\) into
\((\varphi_R,\rho_R,a_R)\): an adapted-cost contraction rate, a retained tail slope
for terminal reporting, and a metric-level window-composition rule.

    \item \emph{Certification and necessity.}
     Gaussian-smoothed denoising geometry gives
inverse-radius certificates on structured windows, while fixed-height and barrier
examples show that height-only summaries cannot recover the load--reserve
certificate.
\end{itemize}

\FloatBarrier

\subsection{Technique overview}
\label{subsec:tech-overview}
The proof follows the forcing--stability--reporting split: construct the learned-flow metric, transport
residual inputs through it, and convert to \(\Wtwo^2\) only at terminal time.
Scalar isotropy is used in the same-drift reflection step. For two learned copies
coupled by reflection until meeting and synchronously afterwards, the separation
radius \(r_t\) obeys the radial comparison
\[
  d\phi(r_t)
  \le
  \bigl\{2\sigma_-^2\phi''(r_t)
        -\bar\kappa(r_t)r_t\phi'(r_t)\bigr\}\,dt+dM_t
  = L_{\bar\kappa,\sigma_-}\phi(r_t)\,dt+dM_t .
\]
Thus any increasing concave cost satisfying
\(L_{\bar\kappa,\sigma_-}\phi\le -\rho\phi\) contracts the learned semigroup in
\(\mathsf W_\phi\). In the fixed-semigroup limit this recovers the Eberle-type contraction
mechanism; here it serves as the stability engine inside a sampler propagation
interface. The diagonal is absorbing because reflection is used only between two
copies of the same learned diffusion, which are coalesced after meeting.

For fixed \(\bar\kappa\), the remaining problem is one-dimensional coercivity.
Writing
\[
\Phi_{\kappa,\sigma}(r)
=\exp\!\left\{-\frac{1}{2\sigma^2}\int_0^r\kappa(u)u\,du\right\},
\qquad
M_\sigma(\kappa)
=\sup_{r>0}\left(\int_0^r\Phi^{-1}\right)
\left(\int_r^\infty\Phi\right),
\]
the one-sided Hardy inequality gives \(\lambda_H(\kappa,\sigma)\asymp
\sigma^2/M_\sigma(\kappa)\). The load--reserve theorem then bounds this capacity
using only
\[
A_R=\int_0^R[-\kappa(r)]_+r\,dr,
\qquad
m_R=\essinf_{r\ge R}\kappa(r),
\]
which yields the affine-tail certificate
\[
\rho_R\gtrsim
\left(m_R\wedge\frac{\sigma^2}{R^2}\right)e^{-A_R/(2\sigma^2)},
\qquad
 a_R\ge e^{-A_R/(2\sigma^2)-1/2}.
\]
The rate \(\rho_R\) is used for adapted-cost propagation; the surviving slope
\(a_R\) records how much large-distance transportation cost remains available for
terminal reporting.

The perturbation step uses this certificate without rebuilding geometry. A
Duhamel/duality argument gives
\[
  \mathsf W_{\phi_R}(\widetilde\mu_T,\mu_T)
  \le
  e^{-\rho_RT}\mathsf W_{\phi_R}(\widetilde\mu_0,\mu_0)
  +
  \int_0^T e^{-\rho_R(T-t)}(\eta_t+\zeta_t)\,dt .
\]
Here the additive appearance of \(\zeta_t\) is a statement about the final
perturbation inequality: local geometry, stiffness, or step-size information may
enter the solver envelope before it is inserted. Terminal tails, moments, or
bounded support are used only afterward to convert the affine-tail estimate into
a \(\Wtwo^2\) report. Profile supply is separate from propagation; the main text
uses the analytic Gaussian-smoothing route, while compact and localized
verification routes are deferred to Appendix~\ref{app:auxiliary-diffusion}.

\section{Related work}
\label{sec:related-work}

We organize related work by the object used to close the continuous propagation
step. Score-estimation bounds, solver analyses, and tail or functional-inequality
estimates can supply forcing, discretization, and terminal reporting inputs. Our
module takes these inputs together with a certified learned-drift radial profile
and transports mismatch in the affine-tail cost determined by that profile.

\paragraph{Global summaries and modular diffusion analyses.}
Many direct \(\Wtwo\)-type guarantees close reverse-time propagation by
controlling the learned drift through a global derivative summary, such as a
Lipschitz constant, Hessian bound, one-sided semiconvexity parameter, or uniform
dissipativity
\citep{chen2022sampling,chen2023improved,lee2023general,gao2025general_w2,bruno2025semiconvex}.
Such summaries are useful when small and verifiable; in our language, uniform
dissipativity is the zero-load case. The load--reserve certificate is
complementary: when a pairwise profile is available, it retains the width and
placement of adverse radial regions instead of compressing them into one scalar
scale. KL/Girsanov, weak-regularity, score-estimation, solver, and
terminal-conversion results may supply the inputs that the profile-based
propagation theorem transports.

\paragraph{Reflection coupling and adapted contraction.}
Reflection coupling is classical
\citep{lindvall1986reflection,chen1989coupling}, and Eberle-type methods
construct concave contraction metrics from radial curvature profiles
\citep{eberle2016reflection,eberle2019couplings}; related geometry-adapted
contraction ideas include non-log-concave Langevin and matrix-metric approaches
\citep{majka2020nonasymptotic,monmarche2023almost}. We use this
reflection-radialization and concave-cost machinery as one component of a sampler
error-propagation interface for learned reverse SDEs. Classical radial-coupling
results prove contraction of a fixed Markov semigroup in an adapted metric. Here
the learned drift supplies the profile used to build the metric, ideal--learned
score mismatch and numerical residuals enter later as residual inputs, and the
affine tail is retained for terminal \(\Wtwo^2\) reporting. Additional
comparisons with Eberle-type contraction, modular \(W_2\) decompositions, solver
analyses, weak-regularity results, and certification interfaces are deferred to
Appendix~\ref{app:additional-related-work}.

\section{Propagation inputs and radial coercivity}
\label{sec:preliminaries}

This section defines the forcing residual and the certified radial profile used
by the theorem, and then introduces the one-dimensional Hardy object used in the
proof. Forcing is measured under the ideal law, while geometry is a pairwise certificate for the learned drift.

\subsection{Reverse diffusion and the forcing residual}

We write sampler time as \(t\in[0,T]\) and noise time as \(s=T-t\). By the standard time-reversal formulas for diffusions \citep{anderson1982reverse,haussmann1986time,song2021sde}, the ideal and learned reverse processes are written abstractly as
\begin{equation}
  dY_t=b_t(Y_t)\,dt+\sigma_t\,dB_t,
  \qquad
  d\widehat Y_t=\widehat b_t(\widehat Y_t)\,dt+\sigma_t\,d\widehat B_t,
  \label{eq:reverse-sdes}
\end{equation}
with laws \(\mu_t\) and \(\widehat\mu_t\). Here \(\widehat\mu_t\) is the
continuous learned reverse law without solver forcing. Later, after a discrete
sampler is embedded in continuous time, \(\widetilde\mu_t\) denotes the
implemented law with solver residuals. In a score-based diffusion model,
\[
  b_t(x)=F_s(x)+g(s)^2\nabla\log p_s(x),
  \qquad
  \widehat b_t(x)=F_s(x)+g(s)^2s_\theta(x,s),
  \qquad s=T-t,
\]
where \(F_s\) denotes the schedule-known non-score drift contribution and \(g(s)=\sigma_t\). In the standard VE/VP/sub-VP affine schedules this contribution is linear, but the propagation theorem below does not use linearity of \(F_s\); the decomposition only identifies the theorem-level drift mismatch with score-modeling error. Hence
\[
  e_t(x):=\widehat b_t(x)-b_t(x)
  =
  g(s)^2\{s_\theta(x,s)-\nabla\log p_s(x)\}.
\]

\begin{definition}[Score-modeling residual]
\label{def:score-modeling-residual}
The additive forcing in the propagation estimate is the \(L^2(\mu_t)\) size of this mismatch:
\begin{equation}
  \eta_t:=\left(\int \|e_t(x)\|^2\,\mu_t(dx)\right)^{1/2}.
  \label{eq:eta-def}
\end{equation}
\end{definition}

When \(Y_t\sim p_{T-t}\), \(\eta_t\) is the schedule-weighted population score-modeling discrepancy at noise time \(s=T-t\). In general, \eqref{eq:eta-def} is the theorem-level forcing term under the ideal law, regardless of how it is estimated. It measures injected learned-score error; it does not determine whether the learned reverse drift contracts or expands pairwise separations.

\begin{assumption}[Certified nondegenerate isotropic window]
\label{ass:standing-window}
On each certified time window,
\begin{equation}
  0<\sigma_-\le \sigma_t\le \sigma_+<\infty .
  \label{eq:noise-window}
\end{equation}
For the clean main statement, \(b_t\) and \(\widehat b_t\) are globally Lipschitz on the window, have at most linear growth, and the laws have the moments required by the displayed bounds.
\end{assumption}

The lower bound \(\sigma_->0\) is structural: concavity gives \(2\sigma_t^2\varphi''\le2\sigma_-^2\varphi''\), so the learned-flow reflection contraction is controlled by \(L_{\bar\kappa,\sigma_-}\), and the Hardy/load--reserve rate is evaluated at the same diffusion scale. The Lipschitz, growth, and moment assumptions are only regularity assumptions for the clean statement; they are not propagation multipliers except insofar as they may be used in a separate profile or solver-residual certificate. Appendix~\ref{app:reflection-coupling} handles reflection, concave-cost regularization, and the coalescing diagonal argument for the learned semigroup; Appendix~\ref{app:propagation-reporting} handles the perturbation step that inserts score and solver forcing. Terminal tail assumptions are separate reporting assumptions.

\subsection{The certified radial profile}

The geometric quantity is pairwise radial contraction. For a vector field \(u:\R^d\to\R^d\), define
\begin{equation}
  \kappa_u(r):=
  \inf_{\|x-y\|=r}
  -\frac{\langle u(x)-u(y),x-y\rangle}{\|x-y\|^2},
  \qquad r>0.
  \label{eq:kappa-u-def}
\end{equation}
Positive values mean contraction of all pairs at radius \(r\) in the separation direction; negative values allow radial expansion. 

\begin{definition}[Certified lower profile and load--reserve data]
\label{def:certified-profile}
Let \(I\subset[0,T]\). A Borel function \(\bar\kappa_I:(0,\infty)\to\R\) is a certified lower radial profile for the learned drift on \(I\) if, for a.e. \(t\in I\) and every \(x\ne y\),
\begin{equation}
  \langle \widehat b_t(x)-\widehat b_t(y),x-y\rangle
  \le
  -\bar\kappa_I(\|x-y\|)\|x-y\|^2.
  \label{eq:certified-profile}
\end{equation}
For \(R>0\), define
\begin{equation}
  A_R(I):=\int_0^R[-\bar\kappa_I(r)]_+\,r\,dr,
  \qquad
  m_R(I):=\essinf_{r\ge R}\bar\kappa_I(r).
  \label{eq:load-reserve-data}
\end{equation}
We say that \((I,R,A,m)\) is profile-admissible if \(A_R(I)\le A<\infty\) and \(m_R(I)\ge m>0\).
\end{definition}

The pairwise form in \eqref{eq:certified-profile} is what the coupling argument needs; purely radius-wise a.e. statements are insufficient without a Borel minorant valid for all pairs. Analytic bounds, interval or verified-network bounds, and deterministic compact
finite covers can provide such conservative minorants; empirical binning is a
useful diagnostic for locating candidate load--reserve windows.

\subsection{The radial Hardy object}

Once a certified profile is fixed, reflection coupling leaves a one-dimensional radial generator. The remaining question is coercivity: can an increasing concave cost spend enough slope to cross the adverse core and still retain positive affine slope in the tail?

The quantities in this subsection are proof-scale objects. They identify the one-dimensional bottleneck left by reflection. The theorem-level stability input remains the load--reserve data \((A_R,m_R,R)\), and the object actually propagated is the affine-tail cost constructed from that data.

Let \(\sigma>0\) and assume that \(r\mapsto\kappa(r)r\) is locally integrable at the origin and away from it, equivalently
\[
  \int_0^a |\kappa(u)|u\,du<\infty
  \qquad\text{for every }a<\infty .
\]
We write this as \(\kappa(r)r\in L^1_{\mathrm{loc}}([0,\infty))\). Define
\begin{equation}
\label{eq:radial-hardy-objects}
\begin{gathered}
L_{\kappa,\sigma}\varphi(r):=
2\sigma^2\varphi''(r)-\kappa(r)r\varphi'(r),
\qquad
\Phi_{\kappa,\sigma}(r):=
\exp\!\left[-\frac{1}{2\sigma^2}
\int_0^r\kappa(u)u\,du\right],\\[-0.25em]
M_\sigma(\kappa):=
\sup_{r>0}
\Bigl(\int_0^r\Phi_{\kappa,\sigma}(s)^{-1}\,ds\Bigr)
\Bigl(\int_r^\infty\Phi_{\kappa,\sigma}(s)\,ds\Bigr).
\end{gathered}
\end{equation}
The weak Sturm--Liouville identity
\[
  L_{\kappa,\sigma}\varphi
  =
  \frac{2\sigma^2}{\Phi_{\kappa,\sigma}}
  (\Phi_{\kappa,\sigma}\varphi')'
\]
shows that \(M_\sigma(\kappa)\) is the one-sided Muckenhoupt Hardy bottleneck for functions vanishing at the origin \citep{muckenhoupt1972hardy}. The corresponding Hardy scale is
\begin{equation}
  \lambda_{\mathrm H}(\kappa,\sigma)
  :=
  \inf_f
  \frac{2\sigma^2\int_0^\infty |f'(r)|^2\Phi_{\kappa,\sigma}(r)\,dr}
       {\int_0^\infty f(r)^2\Phi_{\kappa,\sigma}(r)\,dr},
  \label{eq:lambda-h-def}
\end{equation}
where the infimum is over nonzero functions \(f\in AC_{\mathrm{loc}}([0,\infty))\) such that \(f(0)=0\) and \(f\equiv0\) on \([L,\infty)\) for some finite \(L\). For an increasing concave cost \(\varphi\) with \(\varphi(0)=0\), set
\[
  \mathsf W_\varphi(\nu,\xi)
  :=\inf_{\pi\in\Pi(\nu,\xi)}\int\varphi(\|x-y\|)\,\pi(dx\,dy).
\]
We use \(\mathsf W_\varphi\) on pairs for which the displayed integral is finite. In the main theorem this follows from \(\Wtwo(\widetilde\mu_0,\mu_0)<\infty\) at the initial time and from \(\varphi_R(r)\le r\) throughout propagation.
All constants \(c,C\) are universal unless explicitly stated otherwise.
\section{Main results: adapted propagation and $\Wtwo^2$ reporting}
\label{sec:main-results}

We now separate construction of the propagation certificate from its use.  The
object used by the main theorem is the affine-tail triple
\((\varphi_R,\rho_R,a_R)\): the cost propagated by the learned flow, its
contraction rate, and the tail slope retained for terminal \(\Wtwo^2\) reporting.
Reflection and Hardy capacity turn a radial profile into load--reserve data
\((A_R,m_R,R)\), and the load--reserve construction compiles these data into
\((\varphi_R,\rho_R,a_R)\).  Once this certificate is available, score and solver
residuals are transported as additive inputs.

\subsection{Reflection exposes learned-semigroup contraction}

Scalar isotropy is used only in this step. For two copies of the same learned diffusion, reflection makes the difference noise purely radial and removes the transverse It\^o term. The learned drift is then seen only through its projection onto the separation direction. After the two copies meet, they are coupled synchronously so the diagonal is absorbing; Appendix~\ref{app:reflection-coupling} gives the diagonal and nonsmooth-cost details.

\begin{lemma}[Reflection contraction for the learned semigroup]
\label{lem:reflection-radial-comparison}
Consider two copies of the learned diffusion
\[
  dX_t=\widehat b_t(X_t)\,dt+\sigma_t\,dB^X_t,
  \qquad
  dY_t=\widehat b_t(Y_t)\,dt+\sigma_t\,dB^Y_t
\]
on a window satisfying \Cref{ass:standing-window}. Let $\bar\kappa$ be a certified lower profile for $\widehat b_t$ on this window, and let $\varphi$ be increasing, concave, \(1\)-Lipschitz, and satisfy \(\varphi(0)=0\). Under reflection coupling until the meeting time and synchronous coupling afterwards, for $r_t=\|X_t-Y_t\|$,
\[
  d\varphi(r_t)
  \le
  L_{\bar\kappa,\sigma_-}\varphi(r_t)\,dt+dM_t
\]
on the pre-coupling interval, up to localization and concave-cost approximation. In particular, if
\[
  L_{\bar\kappa,\sigma_-}\varphi\le -\rho\varphi
  \qquad\text{a.e. on }(0,\infty),
\]
then the learned evolution family \(\widehat P_{s,t}\) satisfies
\[
  \mathsf W_\varphi(\nu\widehat P_{s,t},\xi\widehat P_{s,t})
  \le
  e^{-\rho(t-s)}\mathsf W_\varphi(\nu,\xi)
  \qquad(0\le s\le t)
\]
for all probability laws for which the displayed costs are finite.
\end{lemma}

This learned-flow contraction is the only place where reflection is used. It is
applied to two copies of the same learned diffusion, so after coalescence the
diagonal is absorbing. Ideal--learned mismatch and solver residuals are inserted
later by perturbing this contracted evolution; no non-coalescing reflection
comparison between different drifts is required.

\subsection{Constructing the affine-tail propagation certificate}

After reflection, certificate construction is one-dimensional: the Hardy theorem
identifies the radial bottleneck, the load--reserve theorem compresses it into
\((A_R,m_R,R)\), and the affine-tail theorem outputs
\((\varphi_R,\rho_R,a_R)\) for \Cref{thm:main-certificate}.

\begin{theorem}[Hardy capacity of the radialized problem]
\label{thm:hardy-completeness}
For every Borel profile $\kappa$ satisfying $\kappa(r)r\in L^1_{\mathrm{loc}}([0,\infty))$ and every $\sigma>0$,
\begin{equation}
  \frac{\sigma^2}{2M_\sigma(\kappa)}
  \le
  \lambda_{\mathrm H}(\kappa,\sigma)
  \le
  \frac{2\sigma^2}{M_\sigma(\kappa)},
  \label{eq:hardy-completeness-main}
\end{equation}
with the convention $1/\infty=0$.
\end{theorem}

This is the one-sided weighted Hardy inequality \citep{muckenhoupt1972hardy}; see Appendix~\ref{app:hardy-capacity}.  The next result compresses this infinite-dimensional capacity into the finite load--reserve data.

\begin{theorem}[Load--reserve compression]
\label{thm:load-reserve}
Let \(\kappa\) be a Borel profile satisfying
\[
  \kappa(r)r\in L^1_{\mathrm{loc}}([0,\infty)).
\]
Let \(A_R\) and \(m_R\) be the load--reserve quantities in \eqref{eq:load-reserve-data}, with \(\bar\kappa_I\) replaced by \(\kappa\). Assume \(A_R<\infty\) and \(m_R>0\), and set \(L_m:=\sigma/\sqrt{m_R}\). Then
\begin{equation}
    \label{eq:load-reserve-capacity-main}
    M_\sigma(\kappa)
  \le
  C\{L_m^2+e^{A_R/(2\sigma^2)}R(R+L_m)\}.
\end{equation}

Consequently,
\begin{equation}
  \lambda_{\mathrm H}(\kappa,\sigma)
  \ge
  c
  \left(
    m_R\wedge \frac{\sigma^2}{R^2}
  \right)
  e^{-A_R/(2\sigma^2)} .
  \label{eq:load-reserve-lambda-main}
\end{equation}
\end{theorem}

The factors in \eqref{eq:load-reserve-lambda-main} have separate roles:
\(m_R\) is tail reserve, \(\sigma^2/R^2\) is the core-crossing scale, and
\(\exp\{-A_R/(2\sigma^2)\}\) is the slope loss paid before reaching reserve.
Thus \(A_R\) records the width-weighted load spent by a concave metric,
which a pointwise curvature scale alone may miss. The next theorem turns this scale into the
affine-tail certificate used for propagation and reporting.

\begin{theorem}[Constructive affine-tail metric]
\label{thm:affine-tail-metric}
Let \(\kappa\) satisfy the assumptions of \Cref{thm:load-reserve}. Then there exists an increasing concave cost \(\varphi_R\), normalized by \(\varphi_R(0)=0\) and \(\varphi_R'(0+)=1\), affine on \([R,\infty)\), such that
\[
  L_{\kappa,\sigma}\varphi_R(r)
  \le
  -\rho_R\varphi_R(r)
  \qquad\text{for a.e. }r>0,
\]
where the rate can be chosen to satisfy
\begin{equation}
  \rho_R
  \ge
  c\left(m_R\wedge\frac{\sigma^2}{R^2}\right)
  \exp\left\{-\frac{A_R}{2\sigma^2}\right\}.
  \label{eq:rho-R-main}
\end{equation}
Its affine tail slope \(a_R:=\lim_{r\to\infty}\varphi_R'(r)\) satisfies
\begin{equation}
  a_R\ge\exp\left\{-\frac{A_R}{2\sigma^2}-\frac12\right\}.
  \label{eq:a-R-main}
\end{equation}
\end{theorem}

This theorem outputs the propagation certificate
\((\varphi_R,\rho_R,a_R)\).  The operator inequality
\[
  L_{\kappa,\sigma}\varphi_R\le -\rho_R\varphi_R
\]
is the learned-flow contraction certificate.  The rate \(\rho_R\) is used in
\(\mathsf W_{\varphi_R}\) propagation, while the tail slope \(a_R\) is saved for
terminal \(\Wtwo^2\) reporting.  The same adverse load that slows contraction
also reduces \(a_R\), reflecting the rate--reporting tradeoff created by the
slope budget.  Appendix~\ref{app:load-reserve-metric} gives the construction and a
tunable variant, and Appendix~\ref{app:sharpness} shows that the exponential
load factor is structural.

\paragraph{Two regimes.}
Uniform dissipativity is the zero-load case: if
\(\bar\kappa(r)\ge m>0\), then \(A_R=0\), and choosing
\(R\asymp \sigma/\sqrt m\) gives \(\rho_R\gtrsim m\) with a constant affine
tail slope. The first finite-load regime is inverse-radius geometry. If
\(\bar\kappa(r)\ge \alpha-\beta/r\) with \(\alpha,\beta>0\), then choosing
\(R=4\beta/\alpha\) gives
\[
m_R\ge \frac{3\alpha}{4},\qquad
A_R\le \frac{\beta^2}{2\alpha},
\]
and hence
\[
\rho_R\gtrsim
\left(\alpha\wedge \frac{\sigma^2\alpha^2}{\beta^2}\right)
\exp\!\left\{-\frac{\beta^2}{4\alpha\sigma^2}\right\},
\qquad
a_R\ge
\exp\!\left\{-\frac{\beta^2}{4\alpha\sigma^2}-\frac12\right\}.
\]
Thus a bounded adverse region is paid through integrated load, while the affine
tail records how much large-distance transportation cost remains available for
terminal reporting.

\subsection{Profile-adapted propagation and terminal reporting}
We now use the certificate rather than construct new geometry.  An implemented
sampler is represented by
\[
  dX_t=\{\widehat b_t(X_t)+u_t^{\rm sol}\}\,dt+\sigma_t\,dB_t,
  \qquad X_t\sim\widetilde\mu_t,
\]
with integrable envelope
\(\zeta_t=(\mathbb E\|u_t^{\rm sol}\|^2)^{1/2}\).
Deriving \(\zeta_t\) may use local geometry, stiffness, or step-size information;
once certified, it enters the contracted-flow perturbation bound additively.
Appendix~\ref{subsec:euler-solver-envelope} gives a basic Euler example.

\begin{theorem}[Profile-adapted propagation]
\label{thm:main-certificate}
Assume \Cref{ass:standing-window}. Let $\bar\kappa$ be a certified lower profile for $\widehat b_t$ on $[0,T]$, with \(\bar\kappa(r)r\in L^1_{\mathrm{loc}}([0,\infty))\), and suppose $A_R<\infty$ and $m_R>0$ for some $R>0$. Construct $\varphi_R$ by \Cref{thm:affine-tail-metric} with $\kappa=\bar\kappa$ and $\sigma=\sigma_-$. If $\eta,\zeta\in L^1([0,T])$ and $\Wtwo(\widetilde\mu_0,\mu_0)<\infty$, then
\begin{equation}
  \mathsf W_{\varphi_R}(\widetilde\mu_T,\mu_T)
  \le
  e^{-\rho_RT}\mathsf W_{\varphi_R}(\widetilde\mu_0,\mu_0)
  +
  \int_0^T e^{-\rho_R(T-t)}(\eta_t+\zeta_t)\,dt .
  \label{eq:wphi-propagation-main}
\end{equation}
Consequently, since $\varphi_R(r)\le r$,
\begin{equation}
  \mathsf W_{\varphi_R}(\widetilde\mu_T,\mu_T)
  \le
  \Gamma_T,
  \qquad
  \Gamma_T
  :=
  e^{-\rho_RT}\Wtwo(\widetilde\mu_0,\mu_0)
  +
  \int_0^T e^{-\rho_R(T-t)}(\eta_t+\zeta_t)\,dt .
  \label{eq:Gamma-main}
\end{equation}
\end{theorem}

\begin{corollary}[Terminal $\Wtwo^2$ reporting]
\label{cor:terminal-w2-reporting}
Under the assumptions of \Cref{thm:main-certificate}, let $a_R$ be the affine-tail slope of $\varphi_R$. If 
\(S_\lambda
:=
\int e^{\lambda\|x\|^2}\,\mu_T(dx)
+
\int e^{\lambda\|x\|^2}\,\widetilde\mu_T(dx)
<\infty,\)
then
\begin{equation}
  \Wtwo^2(\widetilde\mu_T,\mu_T)
  \le
  C\frac{\Gamma_T}{a_R\sqrt\lambda}
  \sqrt{\log\left(e+\frac{a_RS_\lambda}{\Gamma_T\sqrt\lambda}\right)},
  \label{eq:w2-subgaussian-main}
\end{equation}
with the right side interpreted as zero when $\Gamma_T=0$. If instead, for some $q>2$,
\(M_q
  :=
  \int \|x\|^q\,\mu_T(dx)
  +
  \int \|x\|^q\,\widetilde\mu_T(dx)
  <\infty,\)
then
\begin{equation}
  \Wtwo^2(\widetilde\mu_T,\mu_T)
  \le
  C_q
  \left(\frac{\Gamma_T}{a_R}\right)^{\frac{q-2}{q-1}}
  M_q^{\frac{1}{q-1}} .
  \label{eq:w2-polynomial-main}
\end{equation}
If instead $\mu_T$ and $\widetilde\mu_T$ are supported in a common set of diameter at most $D$, then
\begin{equation}
  \Wtwo^2(\widetilde\mu_T,\mu_T)
  \le
  \frac{D}{a_R}\Gamma_T .
  \label{eq:w2-compact-main}
\end{equation}
\end{corollary}

The theorem and corollary use different parts of the same certificate.
Propagation uses \((\varphi_R,\rho_R)\): residual inputs enter through
\(\Gamma_T\), while learned-flow geometry enters through the certificate
constructed from \((A_R,m_R,R)\). Reporting uses \(a_R\), the large-distance
cost still visible after crossing the adverse core. Terminal tail, moment, or
support assumptions provide the final quadratic reporting input: once one of
these inputs is available, the propagated affine-tail bound becomes a
\(\Wtwo^2\) estimate. Window composition is recorded in
Proposition~\ref{prop:window-composition}.

\subsection{When the certificate is supplied by data geometry}
\label{subsec:data-profile}

Gaussian smoothing supplies a concrete route from data geometry to radial
profiles. With isotropic kernel precision \(\tau_s\),
\[
\nabla\log p_s(x)=-\tau_s\{x-m_s(x)\},
\qquad
m_s(x)=\mathbb E[C_s\mid X_s=x],
\]
where \(C_s\) is the center of the smoothing kernel, such as \(C_s=X_0\) in the
additive VE normalization and \(C_s=a_sX_0\) in an affine VP normalization. The
Gaussian term supplies inward pull; the smoothed denoiser is the data-dependent
term whose pairwise expansion must be controlled.

Assume that on a window \(T-I\),
\[
\langle m_s(x)-m_s(y),x-y\rangle
\le \ell_s\|x-y\|^2+D_s\|x-y\|,
\qquad
\|s_\theta(\cdot,s)-\nabla\log p_s\|_\infty\le E_s .
\]
Then, for
\[
\widehat b_t(x)=f(s)x+g(s)^2s_\theta(x,s),
\qquad s=T-t,
\]
the learned drift has the inverse-radius lower profile
\[
\bar\kappa_I(r)\ge \alpha_I-\frac{\beta_I}{r},
\]
where
\[
\alpha_I:=\essinf_{s\in T-I}\{g(s)^2\tau_s(1-\ell_s)-f(s)\},
\qquad
\beta_I:=\esssup_{s\in T-I}g(s)^2(\tau_sD_s+2E_s).
\]
Thus \(\alpha_I>0\) gives
\[
A_R(I)\le \frac{\beta_I^2}{2\alpha_I},
\qquad
m_R(I)\ge \alpha_I-\frac{\beta_I}{R},
\qquad R>\beta_I/\alpha_I .
\]
The data-level meaning is denoising contraction: on the certified window,
Gaussian pull dominates pairwise expansion of the smoothed denoiser, up to
bounded range and learned-score amplitude terms. Uniformly contractive smoothed
laws give zero load. Bounded denoiser range, common-covariance Gaussian mixtures,
and bounded-amplitude transfers give finite inverse-radius load. Appendix~F
records the calculation and optional localized verification routes.

Combining this inverse-radius profile with Theorem~\ref{thm:affine-tail-metric}
gives a complete window certificate. Taking \(R=4\beta_I/\alpha_I\) when
\(\beta_I>0\) yields
\[
\rho_R\gtrsim
\left(\alpha_I\wedge
\frac{\sigma_-^2\alpha_I^2}{\beta_I^2}\right)
\exp\!\left\{-\frac{\beta_I^2}{4\alpha_I\sigma_-^2}\right\},
\qquad
 a_R\ge
\exp\!\left\{-\frac{\beta_I^2}{4\alpha_I\sigma_-^2}-\frac12\right\}.
\]
Thus the dimensionless load
\(\Lambda_I:=\beta_I^2/(4\alpha_I\sigma_-^2)\) controls both propagation and
reporting; \(R\) is chosen windowwise by the same load--reserve tradeoff, and
multiple certified windows compose by Proposition~\ref{prop:window-composition}.

\subsection{Why width matters beyond height}
\label{subsec:sharpness-main}

A height-only summary can miss the quantity that the certificate construction
spends.  The scarce resource is the derivative of the concave cost.  Crossing an
adverse radial region consumes this derivative in proportion to the
width-weighted load
\[
  A_R=\int_0^R[-\kappa(r)]_+r\,dr,
\]
rather than only to the maximum adverse height
\[
  B(\kappa):=\sup_{r>0}[-\kappa(r)]_+ .
\]
Thus equal adverse height and equal eventual reserve can still lead to
different affine-tail propagation certificates.

\begin{proposition}[Profile height does not determine the radial propagation scale]
\label{prop:main-height-separation}
Fix \(\sigma>0\), \(B>0\), and \(m>0\). For every \(\varepsilon>0\), there exist smooth profiles
\[
  \kappa_1,\ \kappa_2:(0,\infty)\to\mathbb R
\]
with the same one-sided height and the same eventual reserve,
\[
  \sup_{r>0}[-\kappa_1(r)]_+
  =
  \sup_{r>0}[-\kappa_2(r)]_+
  =
  B,
\]
and
\[
  \kappa_1(r)\ge m\quad(r\ge R_1),
  \qquad
  \kappa_2(r)\ge m\quad(r\ge R_2)
\]
for some finite \(R_1,R_2\), but whose Hardy scales satisfy
\[
  \lambda_{\mathrm H}(\kappa_1,\sigma)
  \ge c_0(B,m,\sigma)>0,
  \qquad
  \lambda_{\mathrm H}(\kappa_2,\sigma)
  \le \varepsilon .
\]
Thus one-sided profile height, even together with eventual reserve, need not determine the radial coercivity scale.
\end{proposition}

The proposition identifies the information lost by height-only summaries.
Height records the depth of the adverse barrier; the propagation certificate
also needs the length scale over which the concave cost must spend slope.  If a
cutoff \(R\) were known, the crude estimate \(A_R\le B(\kappa)R^2/2\) would
recover a load bound.  The separation shows that height without such a length
scale cannot determine \((\varphi_R,\rho_R,a_R)\).  Appendix~\ref{app:sharpness}
proves the claim, realizes it by smooth exact-score drifts, and shows that the
exponential load factor is structural.

\section{Conclusion}
\label{sec:conclusion}
We developed a profile-to-certificate calculus for Wasserstein propagation in
scalar-isotropic diffusion samplers. Reflection reduces learned-flow stability to
a radial problem; Hardy capacity identifies the slope bottleneck; and finite load
before reserve constructs the affine-tail certificate \((\varphi_R,\rho_R,a_R)\).
The propagated metric is \(\mathsf W_{\varphi_R}\), while quadratic Wasserstein
error is recovered at terminal time from the retained tail slope and terminal
tail, moment, or support information.

The certificate is finer than a global derivative multiplier. Uniform
dissipativity is the zero-load case, while Gaussian-smoothed denoising geometry
supplies finite-load inverse-radius profiles on structured windows. The
fixed-height separation shows that the width-weighted load is not a presentation
artifact: it is the quantity a concave metric spends before reaching contractive
reserve. This identifies the affine-tail load--reserve certificate as the natural
propagation object for these reverse-SDE windows. 
\newpage
\bibliographystyle{plainnat}
\bibliography{refs}

\appendix
\section{Reflection coupling and radial calculus}
\label{app:reflection-coupling}

This appendix proves the geometric reduction used in \Cref{lem:reflection-radial-comparison}.
The only structural input used here is $d-$dimensional isotropic noise.
For increasing concave radial costs, reflection is the extremal co-adapted coupling in the learned-semigroup contraction step: it maximizes radial quadratic variation of the difference process and removes transverse quadratic variation.
Consequently the radial It\^o formula contains the favorable term \(2\sigma_t^2\varphi''\) and no transverse \((d-1)\varphi'/r\) contribution.
The diagonal is handled only for two copies of the same learned diffusion: after meeting, the copies are coalesced and run synchronously, so the diagonal is absorbing. The ideal--learned mismatch and solver residuals are inserted later by the perturbation argument in Appendix~\ref{app:propagation-reporting}, not by a non-coalescing reflection comparison across the diagonal.

\begin{remark}
    The scalar-isotropic assumption is on the diffusion matrix. Guidance terms that
preserve scalar isotropic noise can be absorbed into the learned drift and are
covered once the guided drift satisfies a certified radial profile. Anisotropic,
degenerate, or state-dependent diffusion matrices would instead call for a
matrix- or Riemannian-adapted contraction metric. Matrix-metric contraction results for constant-diffusion systems, including generalized Langevin diffusions, suggest a rigorous extension route: first prove contraction of the learned semigroup in a chosen quadratic or state-dependent metric, then insert score and solver residuals by the same perturbation/duality step used here. Developing this extension would require a different certificate, involving metric derivatives, diffusion-matrix commutators, and possibly guidance-dependent terms rather than the scalar profile \(\bar\kappa\).
\end{remark}

Reflection coupling is classical \citep{lindvall1986reflection,chen1989coupling,eberle2016reflection,eberle2019couplings}.
The regularization arguments below use standard martingale-transform and Tanaka--Meyer/It\^o approximation tools; see, for example, \citet{karatzas1991brownian,revuz1999continuous}.

\subsection{Extremality of reflection for concave radial costs}
\label{app:reflection-extremality}

We first record the infinitesimal calculation that selects reflection from the class of co-adapted couplings.
Fix a separation direction \(e\in\mathbb S^{d-1}\).
A general co-adapted coupling of two scalar Brownian noises with the same marginal variance can be represented infinitesimally as
\[
  d\widehat B_t
  =O_t\,dB_t+(I-O_tO_t^\top)^{1/2}\,dB'_t,
  \qquad \|O_t\|_{\mathrm{op}}\le1,
\]
where \(B'\) is independent of \(B\).
For a fixed matrix \(O\), the covariance of the difference noise \(d\widehat B_t-dB_t\) is
\[
  A_O:=2I-O-O^\top.
\]
Let \(F(z)=\varphi(\|z\|)\), where \(\varphi\) is increasing and concave.
At a point \(z\ne0\), set \(r=\|z\|\) and \(e=z/r\). Then
\[
  \nabla^2F(z)
  =
  \varphi''(r)ee^\top
  +\frac{\varphi'(r)}{r}(I-ee^\top).
\]
The second-order contribution of the coupling is
\[
  \mathcal D_O\varphi(r)
  :=
  \frac{\sigma^2}{2}\operatorname{tr}\{A_O\nabla^2F(z)\}.
\]

\begin{lemma}[Reflection is extremal for concave radial costs]
\label{lem:reflection-extremal-app}
For every contraction \(O\) with \(\|O\|_{\mathrm{op}}\le1\), every \(r>0\), and every increasing concave \(\varphi\),
\[
  \mathcal D_O\varphi(r)\ge 2\sigma^2\varphi''(r).
\]
Equality is attained by the reflection matrix
\[
  \mathcal R_e:=I-2ee^\top.
\]
Thus, among co-adapted scalar-noise couplings, reflection gives the smallest second-order generator contribution for increasing concave radial costs.
\end{lemma}

\begin{proof}
Since \(A_O\) is a covariance matrix, it is positive semidefinite. Moreover
\[
  e^\top A_Oe
  =2(1-e^\top Oe)
  \le4,
\]
because \(e^\top Oe\ge-1\). Therefore
\[
  \mathcal D_O\varphi(r)
  =
  \frac{\sigma^2}{2}
  \left\{
    \varphi''(r)e^\top A_Oe
    +
    \frac{\varphi'(r)}{r}\operatorname{tr}\big[(I-ee^\top)A_O\big]
  \right\}.
\]
The transverse trace is nonnegative, \(\varphi'\ge0\), and \(\varphi''\le0\). Hence
\[
  \mathcal D_O\varphi(r)
  \ge
  \frac{\sigma^2}{2}\{4\varphi''(r)\}
  =2\sigma^2\varphi''(r).
\]
For \(O=\mathcal R_e=I-2ee^\top\), one has \(A_O=4ee^\top\). The radial variance is exactly \(4\sigma^2\), and the transverse trace is zero, so equality holds.
\end{proof}

The lemma explains why synchronous coupling is not useful for the present contraction mechanism: it has no difference noise and hence no radial smoothing term.
Other co-adapted couplings can generate transverse variation, which is unfavorable for increasing radial costs.
Reflection is precisely the coupling that sends the separation direction to its negative and synchronizes all transverse directions; in this restricted sense it is canonical for the scalar-isotropic, increasing-concave-radial comparison used in the paper.

\subsection{Reflected Brownian motion and the difference equation}
\label{app:reflection-difference-equation}

For \(z\ne0\), define
\[
  e(z)=\frac{z}{\|z\|},
  \qquad
  \mathcal R(z)=I-2e(z)e(z)^\top,
\]
and set \(\mathcal R(0)=I\).
The matrix \(\mathcal R(z)\) is orthogonal and symmetric.

\begin{lemma}[Reflected noise is Brownian]
\label{lem:reflected-noise-brownian}
Let \(B_t\) be a \(d\)-dimensional Brownian motion and let \(Z_t\) be continuous and adapted. Define
\begin{equation}
  \widehat B_t:=\int_0^t \mathcal R(Z_s)\,dB_s .
  \label{eq:Bhat-reflection-app}
\end{equation}
Then \(\widehat B_t\) is a Brownian motion in the same filtration.
\end{lemma}

\begin{proof}
The integrand is predictable and bounded, so \(\widehat B\) is a continuous local martingale. Since \(\mathcal R(z)\mathcal R(z)^\top=I\),
\[
  \langle \widehat B^i,\widehat B^j\rangle_t
  =
  \int_0^t
  (\mathcal R(Z_s)\mathcal R(Z_s)^\top)_{ij}\,ds
  =\delta_{ij}t.
\]
L\'evy's characterization gives the claim.
\end{proof}

Consider two scalar-noise diffusions
\[
  d\widehat Y_t=v_t(\widehat Y_t)\,dt+\sigma_t\,d\widehat B_t,
  \qquad
  dY_t=u_t(Y_t)\,dt+\sigma_t\,dB_t,
\]
with \(Z_t=\widehat Y_t-Y_t\) and \(\widehat B\) defined by \eqref{eq:Bhat-reflection-app}.
On \(\{Z_t\ne0\}\),
\[
  (\mathcal R(Z_t)-I)dB_t
  =-2e(Z_t)e(Z_t)^\top dB_t
  =-2e(Z_t)d\beta_t,
  \qquad
  d\beta_t=e(Z_t)^\top dB_t.
\]
Thus, away from the diagonal,
\begin{equation}
  dZ_t
  =\{v_t(\widehat Y_t)-u_t(Y_t)\}\,dt
  -2\sigma_t e(Z_t)\,d\beta_t .
  \label{eq:reflection-difference-equation-app}
\end{equation}
The choice of \(\beta\) on the zero set is immaterial for the pre-diagonal formula. The only diagonal passage used below is the coalescing same-drift construction in \Cref{app:concave-costs-diagonal}.

\subsection{Radial It\^o formula}
\label{app:radial-ito-formula}

\begin{lemma}[Radial It\^o formula under reflection]
\label{lem:smooth-radial-ito-app}
Assume \(u,v\) are smooth and globally Lipschitz and \(\varphi\in C^2((0,\infty))\). Let \(r_t=\|Z_t\|\). On intervals before the diagonal is hit,
\begin{equation}
  d\varphi(r_t)
  =
  \left[
    2\sigma_t^2\varphi''(r_t)
    +
    \varphi'(r_t)
    \frac{\langle v_t(\widehat Y_t)-u_t(Y_t),Z_t\rangle}{r_t}
  \right]dt+dM_t .
  \label{eq:smooth-radial-ito-app}
\end{equation}
\end{lemma}

\begin{proof}
Let \(F(z)=\varphi(\|z\|)\) for \(z\ne0\). With \(r=\|z\|\) and \(e=z/r\),
\[
  \nabla F(z)=\varphi'(r)e,
  \qquad
  \nabla^2F(z)
  =
  \varphi''(r)ee^\top+\frac{\varphi'(r)}{r}(I-ee^\top).
\]
By \eqref{eq:reflection-difference-equation-app}, the covariance matrix of the martingale part of \(Z_t\) is
\[
  4\sigma_t^2 e(Z_t)e(Z_t)^\top.
\]
Therefore the second-order It\^o contribution is
\[
  \frac12\operatorname{tr}\{4\sigma_t^2ee^\top\nabla^2F(z)\}
  =2\sigma_t^2e^\top\nabla^2F(z)e
  =2\sigma_t^2\varphi''(r).
\]
The transverse Hessian term is absent because the difference noise is purely radial. The first-order term is
\[
  \nabla F(Z_t)\cdot\{v_t(\widehat Y_t)-u_t(Y_t)\}
  =
  \varphi'(r_t)
  \frac{\langle v_t(\widehat Y_t)-u_t(Y_t),Z_t\rangle}{r_t}.
\]
Combining these terms proves \eqref{eq:smooth-radial-ito-app}.
\end{proof}

\subsection{Nonsmooth costs and coalescence at the diagonal}
\label{app:concave-costs-diagonal}

The radial It\^o formula above was written for \(C^2\) costs and for positive
separations.  The cost constructed in \Cref{thm:affine-tail-metric} is only
concave and may have a kink at the cutoff \(R\), while the radius \(r=\|x-y\|\)
is singular at the diagonal.  Both issues are harmless for the same-drift
learned coupling.  Kinks of a concave cost contribute nonpositive local-time
terms, and after the two learned copies meet we coalesce them, so the diagonal
is absorbing.

\begin{lemma}[Concave radial costs under reflection]
\label{lem:concave-radial-cost-app}
Let \(X_t,Y_t\) be two copies of the same learned diffusion, coupled by
reflection until their meeting time
\[
  \tau_0:=\inf\{t:\|X_t-Y_t\|=0\},
\]
and synchronously coalesced after \(\tau_0\).  Set
\(r_t=\|X_t-Y_t\|\).  Let \(\varphi\) be increasing, concave, locally
Lipschitz on \([0,\infty)\), with \(\varphi(0)=0\) and
\(0\le \varphi'\le 1\) a.e.  Write the distributional second derivative on
\((0,\infty)\) as
\[
  D^2\varphi=\varphi''_{\rm ac}(r)\,dr+D^2_{\rm s}\varphi,
  \qquad D^2_{\rm s}\varphi\le 0 .
\]
Assume that, for a.e. \(r>0\),
\[
  2\sigma_-^2\varphi''_{\rm ac}(r)
  -\bar\kappa(r)r\varphi'(r)
  \le -\rho\varphi(r).
\]
Then, up to localization,
\[
  \E \varphi(r_t)
  \le e^{-\rho(t-s)}\E \varphi(r_s),
  \qquad 0\le s\le t .
\]
Consequently the same learned-semigroup contraction as in
\Cref{lem:reflection-radial-comparison} holds for concave affine-tail costs.
\end{lemma}

\begin{proof}
Work first on the stopped interval
\[
  [s,t\wedge \tau_\varepsilon],
  \qquad
  \tau_\varepsilon:=\inf\{u:r_u\le\varepsilon\},
\]
so the process stays away from the diagonal.  The one-dimensional
Tanaka--Meyer formula applied to the continuous semimartingale \(r_u\) gives
the usual absolutely continuous It\^o terms plus a local-time term involving
\(D^2_{\rm s}\varphi\).  Since \(\varphi\) is concave,
\(D^2_{\rm s}\varphi\le0\), and this singular contribution is nonpositive.
Thus it can only improve the drift upper bound.

Using the certified radial profile and \(\sigma_u\ge\sigma_-\), the
absolutely continuous part satisfies
\[
  d\varphi(r_u)
  \le
  \bigl(
    2\sigma_-^2\varphi''_{\rm ac}(r_u)
    -\bar\kappa(r_u)r_u\varphi'(r_u)
  \bigr)du
  +dM_u
  \le
  -\rho\varphi(r_u)\,du+dM_u
\]
on the stopped interval.  Applying the integrating factor and taking
expectations gives
\[
  \E\varphi(r_{t\wedge\tau_\varepsilon})
  \le
  e^{-\rho(t-s)}\E\varphi(r_s).
\]
Letting \(\varepsilon\downarrow0\) gives the estimate up to the meeting time.

After \(\tau_0\), the two copies are driven by the same Brownian motion and
kept equal.  Hence \(r_u=0\) and \(\varphi(r_u)=\varphi(0)=0\) for
\(u\ge\tau_0\).  Coalescence therefore creates no positive boundary term, and
the same estimate holds on the full interval.  Standard localization removes
boundedness restrictions.
\end{proof}

\subsection{Proof of the reflection contraction lemma}
\label{app:proof-reflection-radial-comparison}

\begin{proof}[Proof of \Cref{lem:reflection-radial-comparison}]
It suffices to prove the claim for smooth globally Lipschitz learned drift, a smooth increasing concave \(1\)-Lipschitz cost with \(\varphi(0)=0\), and stopping before the meeting time. \Cref{lem:concave-radial-cost-app} and the preceding coalescence discussion remove these auxiliary restrictions.

Let \(X_t\) and \(Y_t\) be two copies of the learned diffusion, coupled by reflection until \(\tau_0=\inf\{t:X_t=Y_t\}\) and synchronously afterwards. Before \(\tau_0\), apply \Cref{lem:smooth-radial-ito-app} with
\[
  v_t=u_t=\widehat b_t .
\]
For \(Z_t=X_t-Y_t\) and \(r_t=\|Z_t\|\),
\[
  d\varphi(r_t)
  =
  \left[
    2\sigma_t^2\varphi''(r_t)
    +
    \varphi'(r_t)
    \frac{\langle \widehat b_t(X_t)-\widehat b_t(Y_t),Z_t\rangle}{r_t}
  \right]dt+dM_t .
\]
The certified profile gives
\[
  \frac{\langle \widehat b_t(X_t)-\widehat b_t(Y_t),Z_t\rangle}{r_t}
  \le
  -\bar\kappa(r_t)r_t .
\]
Since \(\varphi''\le0\) and \(\sigma_t\ge\sigma_-\),
\[
  2\sigma_t^2\varphi''(r_t)
  \le
  2\sigma_-^2\varphi''(r_t).
\]
Combining the two estimates yields, before \(\tau_0\),
\[
  d\varphi(r_t)
  \le
  {2\sigma_-^2\varphi''(r_t)-\bar\kappa(r_t)r_t\varphi'(r_t)}\,dt+dM_t
  =
  L_{\bar\kappa,\sigma_-}\varphi(r_t)\,dt+dM_t .
\]
After \(\tau_0\), coalescence gives \(r_t=0\) and \(\varphi(r_t)=0\), so the stopped-and-coalesced inequality remains valid.

If \(L_{\bar\kappa,\sigma_-}\varphi\le-\rho\varphi\), then for any initial time \(s\) and terminal time \(t\), localization and the integrating factor give
\[
  \mathbb E\varphi(r_t)
  \le
  e^{-\rho(t-s)}\mathbb E\varphi(r_s).
\]
For every initial coupling of laws \((\nu,\xi)\), the terminal pair is a coupling of \((\nu\widehat P_{s,t},\xi\widehat P_{s,t})\). Taking the infimum over the initial coupling proves
\[
  \mathsf W_\varphi(\nu\widehat P_{s,t},\xi\widehat P_{s,t})
  \le
  e^{-\rho(t-s)}\mathsf W_\varphi(\nu,\xi).
\]
\end{proof}

\section{Hardy capacity}
\label{app:hardy-capacity}

This appendix proves \Cref{thm:hardy-completeness}. The analytic statement is classical: it is the one-sided weighted Hardy inequality of Muckenhoupt type \citep{muckenhoupt1972hardy}; see also standard treatments of Hardy-type inequalities \citep{opic1990hardy,kufner2003weighted,mazya2011sobolev}. We include the proof because this is where the radial comparison from Appendix~\ref{app:reflection-coupling} becomes a coercivity statement.

The point is the following. After reflection contraction of the learned semigroup, the multidimensional estimate is governed by the one-dimensional operator
\[
  L_{\kappa,\sigma}\varphi
  =
  2\sigma^2\varphi''-\kappa(r)r\varphi' .
\]
The boundary \(r=0\) is the coupling boundary, where the transportation cost vanishes. Thus the relevant coercivity problem is not an unconstrained spectral gap on the line, but a one-sided Dirichlet Hardy inequality on \((0,\infty)\). The capacity \(M_\sigma(\kappa)\) is exactly the Muckenhoupt bottleneck for this half-line problem.

Throughout this appendix, assume \(\kappa(r)r\in L^1_{\mathrm{loc}}([0,\infty))\), and write
\[
  \Phi(r)=\Phi_{\kappa,\sigma}(r)
  =
  \exp\left\{
  -\frac{1}{2\sigma^2}\int_0^r\kappa(u)u\,du
  \right\}.
\]

\subsection{From the radial generator to a Hardy inequality}

The radial generator has the divergence form
\begin{equation}
  L_{\kappa,\sigma}\varphi
  =
  \frac{2\sigma^2}{\Phi}
  \bigl(\Phi\varphi'\bigr)' .
  \label{eq:hardy-divergence-form-app}
\end{equation}
Thus the natural energy associated with the radial problem is
\[
  2\sigma^2\int_0^\infty |f'(r)|^2\Phi(r)\,dr,
\]
and the natural mass is
\[
  \int_0^\infty f(r)^2\Phi(r)\,dr.
\]
The condition \(f(0)=0\) encodes the fact that the distance cost vanishes when the two coupled particles meet.

\begin{lemma}[Radial coercivity and the Hardy constant]
\label{lem:hardy-coercivity-equivalence}
Let \(C_H\) be the best constant in
\begin{equation}
  \int_0^\infty f(r)^2\Phi(r)\,dr
  \le
  C_H
  \int_0^\infty |f'(r)|^2\Phi(r)\,dr,
  \label{eq:one-sided-hardy}
\end{equation}
over nonzero functions \(f\in AC_{\mathrm{loc}}([0,\infty))\) satisfying \(f(0)=0\) and \(f\equiv0\) on \([L,\infty)\) for some finite \(L\). Then
\[
  \lambda_{\mathrm H}(\kappa,\sigma)=\frac{2\sigma^2}{C_H},
\]
with the convention \(2\sigma^2/\infty=0\).
\end{lemma}

\begin{proof}
This is just the variational definition \eqref{eq:lambda-h-def}. The numerator in \(\lambda_{\mathrm H}\) is \(2\sigma^2\) times the Dirichlet energy in \eqref{eq:one-sided-hardy}, so the reciprocal of the best Hardy constant gives the Rayleigh quotient.
\end{proof}

This lemma explains why Hardy, rather than an arbitrary Poincaré inequality, appears here: the same-drift radial coupling has an absorbing boundary at the origin and a one-sided tail. The next lemma identifies the best constant in \eqref{eq:one-sided-hardy}, up to the sharp universal factor \(4\), by the Muckenhoupt capacity.

\subsection{The one-sided Muckenhoupt inequality}

\begin{lemma}[One-sided Muckenhoupt Hardy inequality]
\label{lem:muckenhoupt-hardy}
Let \(w:(0,\infty)\to(0,\infty)\) be measurable with \(w,w^{-1}\in L^1_{\mathrm{loc}}([0,\infty))\). Define
\[
  U(r)=\int_0^r w(s)^{-1}\,ds,
  \qquad
  V(r)=\int_r^\infty w(s)\,ds,
  \qquad
  M=\sup_{r>0}U(r)V(r).
\]
Let \(C_H\) be the best constant in
\begin{equation}
  \int_0^\infty f(r)^2w(r)\,dr
  \le
  C_H\int_0^\infty |f'(r)|^2w(r)\,dr,
  \tag{\ref{eq:one-sided-hardy}}
\end{equation}
over nonzero functions \(f\in AC_{\mathrm{loc}}([0,\infty))\) satisfying \(f(0)=0\) and \(f\equiv0\) on \([L,\infty)\) for some finite \(L\). Then
\begin{equation}
  M\le C_H\le 4M,
  \label{eq:muckenhoupt-bound}
\end{equation}
with the convention \(C_H=\infty\) if \(M=\infty\).
\end{lemma}

\begin{proof}
If \(M=\infty\), the lower-bound construction below, with \(r\) chosen along a sequence for which \(U(r)V(r)\to\infty\), gives \(C_H=\infty\). Hence assume \(M<\infty\).

For the lower bound, fix \(r>0\) and \(\varepsilon\in(0,1)\). Choose \(b>r\) such that
\[
  \int_r^b w\ge (1-\varepsilon)V(r).
\]
Since \(M<\infty\) and \(V(r)>0\), the tail integrability of \(w\) implies \(U(\infty)=\infty\). Hence choose \(c>b\) such that
\[
  U(c)-U(b)\ge U(r)/\varepsilon.
\]
Define
\[
f(s)=U(s\wedge r)\quad(s\le b),
\]
\[
f(s)=U(r)\frac{U(c)-U(s)}{U(c)-U(b)}
\quad(b\le s\le c),
\]
and \(f(s)=0\) for \(s\ge c\). Then
\[
  \int |f'|^2w
  \le
  U(r)+\frac{U(r)^2}{U(c)-U(b)}
  \le
  (1+\varepsilon)U(r),
\]
while
\[
  \int f^2w
  \ge
  U(r)^2\int_r^b w
  \ge
  (1-\varepsilon)U(r)^2V(r).
\]
Letting \(\varepsilon\downarrow0\) yields \(C_H\ge U(r)V(r)\). Taking the supremum over \(r\) gives \(C_H\ge M\).

For the upper bound, fix \(\beta\in(0,1)\). Since \(f(0)=0\), Cauchy's inequality gives
\begin{align*}
|f(x)|^2
&\le
\left(\int_0^x |f'(s)|^2w(s)U(s)^\beta\,ds\right)
\left(\int_0^x w(s)^{-1}U(s)^{-\beta}\,ds\right)\\
&=
\frac{U(x)^{1-\beta}}{1-\beta}
\int_0^x |f'(s)|^2w(s)U(s)^\beta\,ds .
\end{align*}
Integrating against \(w(x)dx\) and applying Tonelli,
\[
\int f^2w
\le
\frac{1}{1-\beta}
\int_0^\infty |f'(s)|^2w(s)U(s)^\beta
\left(\int_s^\infty U(x)^{1-\beta}w(x)\,dx\right)ds.
\]
Because \(U(x)V(x)\le M\) and \(w(x)dx=-dV(x)\), integration by parts gives
\[
\int_s^\infty U(x)^{1-\beta}w(x)\,dx
\le
\frac{M}{\beta}U(s)^{-\beta}.
\]
Therefore
\[
  \int f^2w
  \le
  \frac{M}{\beta(1-\beta)}
  \int |f'|^2w.
\]
Taking \(\beta=1/2\) gives \(C_H\le4M\).
\end{proof}

\subsection{Proof of the Hardy capacity theorem}

\begin{proof}[Proof of \Cref{thm:hardy-completeness}]
Apply \Cref{lem:muckenhoupt-hardy} with \(w=\Phi_{\kappa,\sigma}\). Then
\[
  M=M_\sigma(\kappa),
\]
and by \Cref{lem:hardy-coercivity-equivalence},
\[
  \lambda_{\mathrm H}(\kappa,\sigma)=\frac{2\sigma^2}{C_H}.
\]
Since
\[
  M_\sigma(\kappa)\le C_H\le4M_\sigma(\kappa),
\]
inversion gives
\[
  \frac{\sigma^2}{2M_\sigma(\kappa)}
  \le
  \lambda_{\mathrm H}(\kappa,\sigma)
  \le
  \frac{2\sigma^2}{M_\sigma(\kappa)}.
\]
If \(M_\sigma(\kappa)=\infty\), then \(C_H=\infty\) and \(\lambda_{\mathrm H}=0\), which is the same statement under the convention \(1/\infty=0\).
\end{proof}

Thus the capacity \(M_\sigma(\kappa)\) is not an auxiliary artifact. It is the Muckenhoupt scale of the one-dimensional obstruction left after reflection radialization: if \(M_\sigma(\kappa)\) is large, a concave cost must spend too much slope before it can exploit tail contraction; if \(M_\sigma(\kappa)\) is finite, the radial problem has a positive Hardy scale.

\subsection{Structural properties used by load--reserve compression}

The main text does not use the full profile \(\kappa\) directly. It compresses \(M_\sigma(\kappa)\) into load, reserve, and cutoff. The next elementary facts explain why that compression has the right monotonicity and scaling.

\begin{lemma}[Structure and scaling]
\label{lem:hardy-structure-scaling}
The Hardy capacity and Hardy scale have the following properties.
\begin{enumerate}[leftmargin=1.35em,itemsep=1pt,topsep=2pt]
\item If \(\kappa_1\le \kappa_2\) a.e., then
\[
  M_\sigma(\kappa_1)\ge M_\sigma(\kappa_2).
\]
\item For every \(c>0\),
\begin{equation}
  M_\sigma(c\kappa)=M_{\sigma/\sqrt c}(\kappa).
  \label{eq:M-scaling-correct}
\end{equation}
\item For every \(c>0\),
\begin{equation}
  \lambda_{\mathrm H}(c\kappa,\sigma)
  =
  c\,\lambda_{\mathrm H}(\kappa,\sigma/\sqrt c).
  \label{eq:lambda-scaling-correct}
\end{equation}
\end{enumerate}
\end{lemma}

\begin{proof}
Let
\[
  V_i(r)=\int_0^r\kappa_i(u)u\,du.
\]
For \(0\le s\le r\le t\),
\begin{equation}
  \left(\int_0^r\Phi^{-1}\right)
  \left(\int_r^\infty\Phi\right)
  =
  \int_0^r\int_r^\infty
  \exp\left\{\frac{V(s)-V(t)}{2\sigma^2}\right\}\,dt\,ds.
  \label{eq:double-integral-M}
\end{equation}
If \(\kappa_1\le\kappa_2\), then
\[
  V_1(s)-V_1(t)
  =
  -\int_s^t\kappa_1(u)u\,du
  \ge
  -\int_s^t\kappa_2(u)u\,du
  =
  V_2(s)-V_2(t).
\]
The integrand in \eqref{eq:double-integral-M} is therefore larger for \(\kappa_1\), and monotonicity follows after taking the supremum over \(r\).

For scaling,
\[
\Phi_{c\kappa,\sigma}(r)
=
\exp\left\{-\frac{1}{2\sigma^2}
\int_0^r c\kappa(u)u\,du\right\}
=
\Phi_{\kappa,\sigma/\sqrt c}(r).
\]
This proves \eqref{eq:M-scaling-correct}. The same identity of weights gives
\begin{align*}
\lambda_{\mathrm H}(c\kappa,\sigma)
&=
\inf_f
\frac{2\sigma^2\int |f'|^2\Phi_{\kappa,\sigma/\sqrt c}}
     {\int f^2\Phi_{\kappa,\sigma/\sqrt c}}\\
&=
c\inf_f
\frac{2(\sigma^2/c)\int |f'|^2\Phi_{\kappa,\sigma/\sqrt c}}
     {\int f^2\Phi_{\kappa,\sigma/\sqrt c}}
=
c\,\lambda_{\mathrm H}(\kappa,\sigma/\sqrt c).
\end{align*}
\end{proof}

The monotonicity states that making the profile more contractive can only improve the Hardy scale. The scaling identities show that increasing profile strength is equivalent to decreasing the diffusion scale. These are the structural facts used implicitly when the load--reserve theorem balances tail reserve \(m_R\), core-crossing scale \(\sigma^2/R^2\), and adverse load \(A_R\).
\section{Load--reserve compression and the affine-tail metric}
\label{app:load-reserve-metric}

This appendix proves \Cref{thm:load-reserve,thm:affine-tail-metric}. The two proofs should be read as one mechanism. The Hardy capacity from \Cref{app:hardy-capacity} is the exact one-dimensional bottleneck left by reflection contraction of the learned semigroup, but it depends on the whole profile. The load--reserve estimate compresses this bottleneck into the finite data \((A_R,m_R,R)\). The concavity budget then explains why \(A_R\) appears exponentially: an increasing concave cost must spend slope to cross the adverse core. Finally, the affine-tail construction realizes this budget while keeping a positive tail slope for terminal \(\Wtwo^2\) reporting.

\subsection{Compressing the Hardy capacity}

\begin{proof}[Proof of \Cref{thm:load-reserve}]
Write
\[
  V(r)=\int_0^r\kappa(u)u\,du,
  \qquad
  \Phi(r)=e^{-V(r)/(2\sigma^2)},
  \qquad
  A=A_R,
  \qquad
  m=m_R,
  \qquad
  L_m=\frac{\sigma}{\sqrt m}.
\]
For fixed \(r>0\), the Muckenhoupt product can be written as
\begin{equation}
  \left(\int_0^r\Phi^{-1}\right)
  \left(\int_r^\infty\Phi\right)
  =
  \int_0^r\int_r^\infty
  \exp\left\{\frac{V(s)-V(t)}{2\sigma^2}\right\}\,dt\,ds.
  \label{eq:load-reserve-double-int}
\end{equation}
We bound the right-hand side uniformly in \(r\). The proof separates the same two regions that appear in the theorem: the adverse core \([0,R]\) and the reserve tail \([R,\infty)\).

First suppose \(r\le R\). On \(0\le s\le r\le t\le R\), the negative part of \(\kappa(u)u\) in the core has total mass at most \(A\), hence
\[
  V(s)-V(t)
  =
  -\int_s^t\kappa(u)u\,du
  \le
  \int_0^R[-\kappa(u)]_+u\,du
  =
  A.
\]
This core--core region contributes at most
\[
  e^{A/(2\sigma^2)}R^2.
\]
On \(0\le s\le r\le R\le t\), the core again costs at most \(A\), while the tail reserve gives
\[
  V(s)-V(t)
  \le
  A-\int_R^t m u\,du
  =
  A-\frac{m}{2}(t^2-R^2).
\]
Therefore
\[
\int_0^r\int_R^\infty
  \exp\left\{\frac{V(s)-V(t)}{2\sigma^2}\right\}\,dt\,ds
  \le
  e^{A/(2\sigma^2)}R
  \int_R^\infty e^{-m(t^2-R^2)/(4\sigma^2)}\,dt.
\]
Since \(t^2-R^2\ge (t-R)^2\) for \(t\ge R\), the last integral is bounded by \(C L_m\). Thus, when \(r\le R\), the Muckenhoupt product is bounded by
\[
  C e^{A/(2\sigma^2)}R(R+L_m).
\]

Now suppose \(r\ge R\). Split the \(s\)-integral in \eqref{eq:load-reserve-double-int} into \(s\in[0,R]\) and \(s\in[R,r]\). For \(s\in[0,R]\) and \(t\ge r\), the same core-tail estimate gives
\[
  V(s)-V(t)
  \le
  A-\frac{m}{2}(t^2-R^2),
\]
and hence this region contributes at most
\[
  e^{A/(2\sigma^2)}R
  \int_r^\infty e^{-m(t^2-R^2)/(4\sigma^2)}\,dt
  \le
  C e^{A/(2\sigma^2)}R L_m.
\]
For \(R\le s\le r\le t\), the tail reserve gives directly
\[
  V(s)-V(t)
  \le
  -\int_s^t m u\,du
  =
  -\frac{m}{2}(t^2-s^2).
\]
Since \(t^2-s^2\ge (t-s)^2\), and with \(u=r-s\), \(v=t-r\), we get
\begin{align*}
\int_R^r\int_r^\infty
  e^{-m(t^2-s^2)/(4\sigma^2)}\,dt\,ds
  &\le
  \int_0^\infty\int_0^\infty
  e^{-m(u+v)^2/(4\sigma^2)}\,dv\,du  \\
  &=
  \int_0^\infty w e^{-mw^2/(4\sigma^2)}\,dw
  \le
  C L_m^2 .
\end{align*}
Taking the supremum over \(r\) in \eqref{eq:load-reserve-double-int} gives
\[
  M_\sigma(\kappa)
  \le
  C\left[L_m^2+e^{A/(2\sigma^2)}R(R+L_m)\right],
\]
which proves \eqref{eq:load-reserve-capacity-main}.

It remains to convert this capacity bound into the displayed Hardy-scale lower bound. Let
\[
  E_A:=e^{A/(2\sigma^2)}.
\]
Since \(E_A\ge1\),
\[
  L_m^2+E_A R(R+L_m)
  \le
  C E_A(R+L_m)^2.
\]
By \Cref{thm:hardy-completeness},
\[
  \lambda_{\mathrm H}(\kappa,\sigma)
  \ge
  \frac{\sigma^2}{2M_\sigma(\kappa)}
  \ge
  c E_A^{-1}\frac{\sigma^2}{(R+L_m)^2}.
\]
Finally,
\[
  \frac{\sigma^2}{(R+L_m)^2}
  \ge
  c\left(\frac{\sigma^2}{R^2}\wedge\frac{\sigma^2}{L_m^2}\right)
  =
  c\left(\frac{\sigma^2}{R^2}\wedge m_R\right).
\]
Combining the last two displays proves \eqref{eq:load-reserve-lambda-main}.
\end{proof}

\subsection{The slope budget of an adverse core}

The previous proof shows that \(A_R\) controls the Hardy capacity. The next lemma explains the same quantity from the metric side. If the profile is adverse on the core, a concave cost must spend derivative there; the total logarithmic loss of slope is exactly the integrated load.

\begin{lemma}[Concavity budget]
\label{lem:concavity-budget}
Let \(q\ge0\) and \(\int_0^R q(r)r\,dr<\infty\). Suppose \(\varphi\) is increasing and concave, \(\varphi'\) is absolutely continuous on \((0,R)\), and
\begin{equation}
  2\sigma^2\varphi''(r)+q(r)r\varphi'(r)\le0
  \qquad\text{for a.e. }0<r<R.
  \label{eq:concavity-budget-ineq}
\end{equation}
Then
\begin{equation}
  \log\frac{\varphi'(0+)}{\varphi'(R-)}
  \ge
  \frac{1}{2\sigma^2}\int_0^R q(r)r\,dr.
  \label{eq:exact-budget}
\end{equation}
The bound is sharp: equality is attained by
\[
  \varphi'(r)
  =
  \varphi'(0+)
  \exp\left\{-\frac{1}{2\sigma^2}\int_0^r q(u)u\,du\right\}.
\]
\end{lemma}

\begin{proof}
If \(\varphi'\) vanishes before reaching \(R\), then the left-hand side of \eqref{eq:exact-budget} is infinite and there is nothing to prove. Otherwise \(\varphi'(r)>0\) on \((0,R)\), so we can divide \eqref{eq:concavity-budget-ineq} by \(\varphi'(r)\) and integrate:
\[
  2\sigma^2\int_0^R \frac{\varphi''(r)}{\varphi'(r)}\,dr
  +
  \int_0^R q(r)r\,dr
  \le0.
\]
Since \((\log\varphi')'=\varphi''/\varphi'\) a.e.,
\[
  \log\varphi'(R-)-\log\varphi'(0+)
  \le
  -\frac{1}{2\sigma^2}\int_0^R q(r)r\,dr.
\]
This is \eqref{eq:exact-budget}. The displayed extremizer is positive, nonincreasing, absolutely continuous, and satisfies \eqref{eq:concavity-budget-ineq} with equality.
\end{proof}

\subsection{Constructing the affine-tail metric}

We now turn the load--reserve data into the cost used in the propagation theorem. The construction has three requirements. It must spend concavity inside the adverse core, contract at the core-crossing scale, and keep a positive affine tail for terminal \(\Wtwo^2\) reporting.

\begin{proof}[Proof of \Cref{thm:affine-tail-metric}]
Let
\[
  q=[-\kappa]_+,\qquad A=A_R,\qquad m=m_R,\qquad b=\frac{\sigma^2}{R}.
\]
Define
\[
  H(r)=\int_0^{r\wedge R}\frac{q(u)u+b}{2\sigma^2}\,du,
  \qquad
  h(r)=e^{-H(r)},
\]
and
\[
  \varphi_R(r)
  =
  \int_0^{r\wedge R}h(s)\,ds+h(R)(r-R)_+.
\]
Then \(\varphi_R\) is increasing and concave, \(\varphi_R(0)=0\), \(\varphi_R'(0+)=1\), and \(\varphi_R\) is affine on \([R,\infty)\). Moreover
\[
  H(R)=\frac{A}{2\sigma^2}+\frac12,
\]
so its affine tail slope is
\[
  a_R=h(R)=\exp\left\{-\frac{A}{2\sigma^2}-\frac12\right\},
\]
which proves \eqref{eq:a-R-main}.

We prove the generator inequality separately on the core and the tail.

\paragraph{Core: \(0<r<R\).}
For a.e. \(0<r<R\),
\[
  \varphi_R'(r)=h(r),
  \qquad
  \varphi_R''(r)=-\frac{q(r)r+b}{2\sigma^2}\,\varphi_R'(r).
\]
Thus
\[
  2\sigma^2\varphi_R''(r)+q(r)r\varphi_R'(r)
  =
  -b\varphi_R'(r).
\]
Since \(-\kappa(r)r\le q(r)r\), we obtain
\[
  L_{\kappa,\sigma}\varphi_R(r)
  =
  2\sigma^2\varphi_R''(r)-\kappa(r)r\varphi_R'(r)
  \le
  -b\varphi_R'(r).
\]
On the core, \(H(r)\le H(R)\), hence \(\varphi_R'(r)\ge a_R\). Also \(\varphi_R(r)\le r\le R\). Therefore
\[
  b\varphi_R'(r)
  \ge
  \frac{\sigma^2}{R}a_R
  \ge
  \frac{\sigma^2}{R^2}a_R\,\varphi_R(r).
\]
Consequently,
\[
  L_{\kappa,\sigma}\varphi_R(r)
  \le
  -\frac{\sigma^2}{R^2}a_R\,\varphi_R(r),
  \qquad 0<r<R.
\]

\paragraph{Tail: \(r>R\).}
On the tail, \(\varphi_R'(r)=a_R\) and \(\varphi_R''(r)=0\). Since \(m_R=m\), we have \(\kappa(r)\ge m\) for a.e. \(r\ge R\). Thus
\[
  L_{\kappa,\sigma}\varphi_R(r)
  =
  -\kappa(r)r a_R
  \le
  -m r a_R.
\]
Furthermore,
\[
  \varphi_R(r)
  =
  \varphi_R(R)+a_R(r-R)
  \le
  R+a_R(r-R),
\]
and because \(0<a_R\le1\),
\[
  a_R r
  =
  a_R R+a_R(r-R)
  \ge
  a_R R+a_R^2(r-R)
  \ge
  a_R\varphi_R(r).
\]
Therefore
\[
  L_{\kappa,\sigma}\varphi_R(r)
  \le
  -m a_R\,\varphi_R(r),
  \qquad r>R.
\]

Combining the core and tail estimates gives
\[
  L_{\kappa,\sigma}\varphi_R(r)
  \le
  -a_R\left(m\wedge\frac{\sigma^2}{R^2}\right)\varphi_R(r)
  \qquad\text{for a.e. }r>0.
\]
Since
\[
  a_R=\exp\left\{-\frac{A}{2\sigma^2}-\frac12\right\}
  \ge
  e^{-1/2}\exp\left\{-\frac{A}{2\sigma^2}\right\},
\]
the claimed rate \eqref{eq:rho-R-main} follows after decreasing the universal constant \(c\). The possible kink at \(R\) is irrelevant for the a.e. generator inequality and is handled by the concave-cost approximation in Appendix~\ref{app:reflection-coupling}.
\end{proof}

\paragraph{Tunable core-bending variant.}
The same construction gives a one-parameter family useful for optimizing the rate--reporting tradeoff. For \(\gamma>0\), replace \(b=\sigma^2/R\) by
\[
  b_\gamma=\frac{\gamma\sigma^2}{R}
\]
and define
\[
  H_\gamma(r)=\int_0^{r\wedge R}\frac{q(u)u+b_\gamma}{2\sigma^2}\,du,
  \qquad
  \varphi_{\gamma,R}(r)
  =
  \int_0^{r\wedge R}e^{-H_\gamma(s)}\,ds
  +
  e^{-H_\gamma(R)}(r-R)_+ .
\]
Then
\[
  a_{\gamma,R}:=\lim_{r\to\infty}\varphi_{\gamma,R}'(r)
  =
  \exp\left\{-\frac{A_R}{2\sigma^2}-\frac{\gamma}{2}\right\}.
\]
Repeating the core estimate above with \(b_\gamma\) gives
\[
  L_{\kappa,\sigma}\varphi_{\gamma,R}(r)
  \le
  -a_{\gamma,R}
  \left(m_R\wedge\frac{\gamma\sigma^2}{R^2}\right)
  \varphi_{\gamma,R}(r)
  \qquad\text{for a.e. }r>0.
\]
Thus \(\gamma\) increases the core-crossing contraction scale but decreases the affine tail slope. The main theorem uses \(\gamma=1\) for notational simplicity.
\paragraph{Rate--reporting tradeoff.}
The parameter \(\gamma\) makes explicit that propagation rate and terminal
quadratic reporting cannot be optimized independently. Increasing \(\gamma\)
improves the core-crossing factor
\(\gamma\sigma^2/R^2\), but decreases the surviving affine-tail slope
\(a_{\gamma,R}\). Since the terminal \(W_2^2\) reports in Corollary 1 scale
through \(a_{\gamma,R}^{-1}\), over-bending the cost can make the learned-flow
contraction easier to prove while making quadratic reporting more expensive.
Thus the adverse load affects the theorem twice: it slows propagation through
the contraction certificate and reduces the amount of large-distance
transportation cost visible at terminal time. This is why the theorem
propagates an adapted affine-tail cost and reports \(W_2^2\) only after
separate terminal tail, moment, or support information is supplied.
\section{Propagation and reporting}
\label{app:propagation-reporting}

This appendix proves \Cref{thm:main-certificate,cor:terminal-w2-reporting} and the window-composition rule used in the main text. At this point the geometry has already been identified: Appendix~\ref{app:reflection-coupling} gives contraction of the learned semigroup in a radial cost, and Appendix~\ref{app:load-reserve-metric} builds a cost \(\varphi_R\) satisfying
\[
  L_{\bar\kappa,\sigma_-}\varphi_R\le -\rho_R\varphi_R .
\]
The remaining work is conceptually different. We no longer search for geometry; we perturb the contracted learned flow by the score and solver residuals, and then perform a deterministic terminal conversion from the affine-tail cost to \(\Wtwo^2\). 

\subsection{Profile-resolved propagation}
\label{subsec:profile-resolved-propagation}

The propagation theorem is a stability statement in the metric already chosen by the load--reserve construction. The certified profile determines \(\varphi_R\) and \(\rho_R\), hence contraction of the learned flow; the score-modeling and solver residuals then enter only as additive perturbations of this contracted flow.

For an increasing concave cost \(\varphi\) with \(\varphi(0)=0\), write
\[
  d_\varphi(x,y):=\varphi(\|x-y\|).
\]
By concavity and subadditivity of \(\varphi\), this is a metric whenever \(\varphi(r)>0\) for \(r>0\), as in the affine-tail costs used below. When \(0\le\varphi'\le1\), every \(d_\varphi\)-Lipschitz test function is Euclidean Lipschitz with the same constant, because \(\varphi(r)\le r\).

\begin{lemma}[Perturbations of a contracted learned flow]
\label{lem:contracted-flow-perturbation-app}
Let \(\widehat P_{s,t}\) be the learned evolution family. Assume that, for some increasing concave \(\varphi\) with \(\varphi(0)=0\), \(\varphi(r)>0\) for \(r>0\), and \(0\le\varphi'\le1\),
\[
  \mathsf W_\varphi(\nu\widehat P_{s,t},\xi\widehat P_{s,t})
  \le
  e^{-\rho(t-s)}\mathsf W_\varphi(\nu,\xi)
  \qquad(0\le s\le t)
\]
for all laws with finite \(\mathsf W_\varphi\)-cost. Let \(\alpha_t\) and \(\beta_t\) solve the Fokker--Planck equations corresponding to the same scalar diffusion coefficient \(\sigma_t\) and drifts \(\widehat b_t+v_t\) and \(\widehat b_t+w_t\), respectively. If
\[
  a_t:=\int \|v_t(x)\|\,\alpha_t(dx),
  \qquad
  b_t:=\int \|w_t(x)\|\,\beta_t(dx)
\]
are integrable on \([0,T]\), then
\[
  \mathsf W_\varphi(\alpha_T,\beta_T)
  \le
  e^{-\rho T}\mathsf W_\varphi(\alpha_0,\beta_0)
  +
  \int_0^T e^{-\rho(T-t)}(a_t+b_t)\,dt .
\]
\end{lemma}

\begin{proof}
We give the argument for smooth coefficients and bounded perturbations; standard localization, mollification, and truncation give the stated form under the regularity assumptions of \Cref{ass:standing-window}. By Kantorovich duality for the metric \(d_\varphi\), it is enough to fix a test function \(f\) with
\[
  |f(x)-f(y)|\le d_\varphi(x,y)
\]
and bound \(\alpha_T(f)-\beta_T(f)\). Let
\[
  \psi_t:=\widehat P_{t,T}f .
\]
The assumed contraction implies that \(\psi_t\) is \(e^{-\rho(T-t)}\)-Lipschitz with respect to \(d_\varphi\). Since \(d_\varphi(x,y)\le\|x-y\|\), it is Euclidean \(e^{-\rho(T-t)}\)-Lipschitz, and hence \(\|\nabla\psi_t\|\le e^{-\rho(T-t)}\) a.e. The backward equation for the learned generator gives
\[
  \frac{d}{dt}\alpha_t(\psi_t)
  =
  \int v_t(x)\cdot\nabla\psi_t(x)\,\alpha_t(dx),
  \qquad
  \frac{d}{dt}\beta_t(\psi_t)
  =
  \int w_t(x)\cdot\nabla\psi_t(x)\,\beta_t(dx).
\]
Therefore
\[
\begin{aligned}
\alpha_T(f)-\beta_T(f)
&=
\alpha_0(\psi_0)-\beta_0(\psi_0)  \\
&\quad+
\int_0^T
\left\{
  \int v_t\cdot\nabla\psi_t\,d\alpha_t
  -
  \int w_t\cdot\nabla\psi_t\,d\beta_t
\right\}dt .
\end{aligned}
\]
The first term is bounded by \(e^{-\rho T}\mathsf W_\varphi(\alpha_0,\beta_0)\), because \(\psi_0\) is \(e^{-\rho T}\)-Lipschitz in \(d_\varphi\). The integral is bounded above by
\[
  \int_0^T e^{-\rho(T-t)}(a_t+b_t)\,dt .
\]
Taking the supremum over admissible \(f\) proves the claim.
\end{proof}

\begin{proof}[Proof of \Cref{thm:main-certificate}]
Let \(\widehat P_{s,t}\) denote the learned evolution family. By \Cref{lem:reflection-radial-comparison} and the operator inequality from \Cref{thm:affine-tail-metric},
\[
  \mathsf W_{\varphi_R}(\nu\widehat P_{s,t},\xi\widehat P_{s,t})
  \le
  e^{-\rho_R(t-s)}\mathsf W_{\varphi_R}(\nu,\xi).
\]
We now apply \Cref{lem:contracted-flow-perturbation-app} to the implemented and ideal laws.

The ideal law \(\mu_t\) solves the learned equation with perturbation
\[
  w_t(x)=b_t(x)-\widehat b_t(x)=-e_t(x).
\]
Thus, by Cauchy--Schwarz inequality,
\[
  \int \|w_t(x)\|\,\mu_t(dx)
  =
  \int \|e_t(x)\|\,\mu_t(dx)
  \le
  \eta_t .
\]
The implemented interpolation may have a path-dependent residual \(u_t^{\rm sol}\). Let
\[
  \bar u_t(x):=\mathbb E[u_t^{\rm sol}\mid X_t=x]
\]
be a measurable conditional mean under the implemented process. Its one-time marginals solve the Fokker--Planck equation with drift \(\widehat b_t+\bar u_t\), and Jensen's inequality gives
\[
  \int \|\bar u_t(x)\|\,\widetilde\mu_t(dx)
  \le
  \mathbb E\|u_t^{\rm sol}\|
  \le
  \zeta_t .
\]
Applying \Cref{lem:contracted-flow-perturbation-app} with
\[
  \alpha_t=\widetilde\mu_t,
  \qquad
  \beta_t=\mu_t,
  \qquad
  v_t=\bar u_t,
  \qquad
  w_t=-e_t,
\]
yields
\[
  \mathsf W_{\varphi_R}(\widetilde\mu_T,\mu_T)
  \le
  e^{-\rho_RT}\mathsf W_{\varphi_R}(\widetilde\mu_0,\mu_0)
  +
  \int_0^T e^{-\rho_R(T-t)}(\eta_t+\zeta_t)\,dt .
\]
Finally, \(\varphi_R(r)\le r\) implies
\[
  \mathsf W_{\varphi_R}(\widetilde\mu_0,\mu_0)
  \le
  \Wone(\widetilde\mu_0,\mu_0)
  \le
  \Wtwo(\widetilde\mu_0,\mu_0),
\]
which gives the bound by \(\Gamma_T\).
\end{proof}

\subsection{A concrete solver-residual envelope}
\label{subsec:euler-solver-envelope}

The propagation theorem treats the solver residual as a modular input. This modularity is a law-level perturbation statement: once an adapted continuous-time embedding has been chosen and an envelope \(\zeta_t\) has been certified, the contracted learned flow propagates it additively. It does not assert that the envelope itself is geometry-free. In practical discretizations, especially with large steps or higher-order updates, the construction of \(\zeta_t\) may contain local Jacobian, Hessian, commutator, moment, or profile-dependent factors. The following elementary interpolation shows how such an envelope arises for a standard drift-frozen Euler implementation. Let \(0=t_0<t_1<\cdots<t_N=T\), \(h_k=t_{k+1}-t_k\), and on \(t\in[t_k,t_{k+1})\) define
\[
  dX_t=\widehat b_{t_k}(X_{t_k})\,dt+\sigma_t\,dB_t .
\]
Relative to the learned continuous SDE \(dX_t=\widehat b_t(X_t)\,dt+\sigma_t\,dB_t\), this interpolation has
\[
  u_t^{\rm sol}
  =
  \widehat b_{t_k}(X_{t_k})-\widehat b_t(X_t).
\]
If \(\widehat b_t\) is \(L_x\)-Lipschitz in space and \(L_t\)-Lipschitz in time on the relevant window or stopped region, then
\[
  \zeta_t
  \le
  L_x\left(\mathbb E\|X_t-X_{t_k}\|^2\right)^{1/2}
  +
  L_t|t-t_k|.
\]
If additionally
\[
  \left(\mathbb E\|\widehat b_{t_k}(X_{t_k})\|^2\right)^{1/2}
  \le B_k ,
\]
then, for \(t\in[t_k,t_{k+1})\),
\[
  \left(\mathbb E\|X_t-X_{t_k}\|^2\right)^{1/2}
  \le
  B_k h_k+\sigma_+\sqrt{d h_k}.
\]
Consequently one admissible envelope is
\[
  \zeta_t
  \le
  L_x\{B_k h_k+\sigma_+\sqrt{d h_k}\}
  +
  L_t h_k,
  \qquad t\in[t_k,t_{k+1}).
\]
This is not intended as a sharp discretization theorem. It only shows how a standard solver analysis can be plugged into the propagation interface. If a sharper analysis yields a residual \(\zeta_t^{R,K}\) depending on the certified window, cutoff, local region, or step-size regime, the same proof applies with that envelope. This is the intended way to account for multiplicative interactions between discretization error and local drift geometry before the residual is propagated additively by the contracted learned flow. Higher-order stochastic reverse-SDE solvers can replace this display with sharper residual envelopes; PF-ODE/DDIM-type samplers require a separate comparison or extension.

\subsection{Terminal reporting from affine-tail costs to \(\Wtwo^2\)}
\label{subsec:terminal-reporting-proof}

The propagation theorem controls an affine-tail transportation cost, not \(\Wtwo^2\) directly. This is intentional. The affine tail gives a stable propagation metric even when the profile is only eventually contractive; terminal information is used only afterward to recover a quadratic cost.

\begin{proof}[Proof of \Cref{cor:terminal-w2-reporting}]
Set
\[
  \mu=\widetilde\mu_T,
  \qquad
  \nu=\mu_T,
  \qquad
  \varphi=\varphi_R,
  \qquad
  a=a_R,
  \qquad
  \Delta=\Gamma_T .
\]
By \eqref{eq:Gamma-main},
\[
  \mathsf W_\varphi(\mu,\nu)\le \Delta .
\]
If \(\Delta=0\), then \(\mathsf W_\varphi(\mu,\nu)=0\). Since the affine tail slope is positive, \(\varphi(r)>0\) for \(r>0\), hence \(\mu=\nu\) and all three bounds are trivial. Assume \(\Delta>0\).

Because \(\varphi\) is increasing and concave and has affine tail slope at least \(a>0\), its derivative is at least \(a\) a.e. Therefore
\begin{equation}
  \varphi(r)\ge ar,
  \qquad r\ge0 .
  \label{eq:phi-controls-linear}
\end{equation}
Let \(\pi\) be any coupling with
\[
  \int\varphi(\|x-y\|)\,d\pi\le \Delta+\varepsilon .
\]
Write \(R(x,y)=\|x-y\|\). For a threshold \(L>0\), split
\[
  \int R^2\,d\pi
  =
  \int_{R\le L}R^2\,d\pi
  +
  \int_{R>L}R^2\,d\pi .
\]
On the first region, \(R^2\le LR\le (L/a)\varphi(R)\), so
\begin{equation}
  \int_{R\le L}R^2\,d\pi
  \le
  \frac{L}{a}(\Delta+\varepsilon).
  \label{eq:reporting-core-term}
\end{equation}

\paragraph{Sub-Gaussian reporting.}
Assume \(S_\lambda<\infty\). Since \(R^2\le2\|x\|^2+2\|y\|^2\), Cauchy--Schwarz inequality gives
\[
  \int e^{\lambda R^2/4}\,d\pi
  \le
  \left(\int e^{\lambda\|x\|^2}\,\mu(dx)\right)^{1/2}
  \left(\int e^{\lambda\|y\|^2}\,\nu(dy)\right)^{1/2}
  \le
  S_\lambda .
\]
For \(u\ge L^2\), \(u\le(8/\lambda)e^{\lambda u/8}\), and hence
\[
  R^2\mathbf 1_{\{R>L\}}
  \le
  \frac{8}{\lambda}
  e^{-\lambda L^2/8}
  e^{\lambda R^2/4}.
\]
Thus
\[
  \int_{R>L}R^2\,d\pi
  \le
  \frac{8S_\lambda}{\lambda}e^{-\lambda L^2/8}.
\]
Choose
\[
  L
  =
  \sqrt{\frac{8}{\lambda}
  \log\left(e+\frac{aS_\lambda}{\Delta\sqrt\lambda}\right)} .
\]
Then \eqref{eq:reporting-core-term} is bounded by
\[
  C\frac{\Delta+\varepsilon}{a\sqrt\lambda}
  \sqrt{\log\left(e+\frac{aS_\lambda}{\Delta\sqrt\lambda}\right)},
\]
and the tail term is bounded by \(C\Delta/(a\sqrt\lambda)\), which is absorbed because the logarithm is at least one. Letting \(\varepsilon\downarrow0\) and taking the infimum over couplings proves \eqref{eq:w2-subgaussian-main}.

\paragraph{Finite-moment reporting.}
Assume \(M_q<\infty\) for some \(q>2\). By \eqref{eq:reporting-core-term}, the core contribution is at most \((L/a)(\Delta+\varepsilon)\). Moreover,
\[
  R^q\le 2^{q-1}(\|x\|^q+\|y\|^q),
\]
so
\[
  \int_{R>L}R^2\,d\pi
  \le
  L^{-(q-2)}\int R^q\,d\pi
  \le
  2^{q-1}M_qL^{-(q-2)} .
\]
Optimizing \(L\) gives
\[
  \int R^2\,d\pi
  \le
  C_q
  \left(\frac{\Delta+\varepsilon}{a}\right)^{\frac{q-2}{q-1}}
  M_q^{\frac{1}{q-1}} .
\]
Letting \(\varepsilon\downarrow0\) and taking the infimum over couplings proves \eqref{eq:w2-polynomial-main}.

\paragraph{Bounded-support reporting.}
If both measures are supported in a common set of diameter \(D\), then \(R\le D\) under every coupling. By \eqref{eq:phi-controls-linear},
\[
  R^2\le DR\le\frac{D}{a}\varphi(R).
\]
Taking the infimum over couplings gives
\[
  \Wtwo^2(\mu,\nu)
  \le
  \frac{D}{a}\mathsf W_\varphi(\mu,\nu)
  \le
  \frac{D}{a}\Delta,
\]
which is \eqref{eq:w2-compact-main}.
\end{proof}

The proof also explains the role of the affine-tail slope. The slope \(a_R\) is not a contraction rate; it is the amount of linear transportation cost still visible at large distances. A smaller \(a_R\) can make propagation easier to build through the adverse core, but makes terminal \(\Wtwo^2\) reporting more expensive. This is why terminal assumptions appear only in the reporting corollary, not in the propagation estimate.

\begin{remark}
    In applications, these inputs may come from target assumptions, path-moment
or Lyapunov estimates for the implemented interpolation, clipping, or localization.
\end{remark}

\begin{proposition}[Sharp order of sub-Gaussian terminal reporting]
\label{prop:reporting-sharpness-app}
The sub-Gaussian reporting implication in \Cref{cor:terminal-w2-reporting} has the correct order under only the affine-tail lower bound \(\varphi(r)\ge ar\) and a terminal sub-Gaussian envelope. In general, the factor
\[
  \frac{\Delta}{a\sqrt\lambda}
  \sqrt{\log\left(e+\frac{aS_\lambda}{\Delta\sqrt\lambda}\right)}
\]
cannot be improved beyond universal constants.
\end{proposition}

\begin{proof}
It suffices to test the affine cost \(\varphi(r)=ar\) in one dimension. Let
\[
  \nu=\delta_0,
  \qquad
  \mu=(1-\theta)\delta_0+\theta\delta_L.
\]
Then
\[
  \mathsf W_\varphi(\mu,\nu)=\theta aL,
  \qquad
  \Wtwo^2(\mu,\nu)=\theta L^2,
\]
and
\[
  S_\lambda(\mu,\nu)=2-\theta+\theta e^{\lambda L^2}.
\]
Fix \(S>3\) and a small target affine-tail cost \(\Delta>0\). Set
\[
  \theta=\frac{\Delta}{aL}
\]
and choose \(L\) as the largest solution of
\begin{equation}
    \frac{\Delta}{aL}e^{\lambda L^2}=S-2.
  \label{eq:reporting-sharpness-implicit}
\end{equation}
Then \(\mathsf W_\varphi(\mu,\nu)=\Delta\) and \(S_\lambda(\mu,\nu)\le S\) for sufficiently small \(\Delta\). Equation \eqref{eq:reporting-sharpness-implicit} implies
\[
  \lambda L^2
  =
  \log\left(\frac{a(S-2)L}{\Delta}\right).
\]
Solving this implicit relation as \(\Delta\downarrow0\) gives
\[
  L
  \asymp
  \lambda^{-1/2}
  \sqrt{\log\left(e+\frac{aS}{\Delta\sqrt\lambda}\right)}.
\]
Therefore
\[
  \Wtwo^2(\mu,\nu)
  =
  \theta L^2
  =
  \frac{\Delta L}{a}
  \asymp
  \frac{\Delta}{a\sqrt\lambda}
  \sqrt{\log\left(e+\frac{aS}{\Delta\sqrt\lambda}\right)}.
\]
This matches the sub-Gaussian bound in \Cref{cor:terminal-w2-reporting} up to universal constants.
\end{proof}
\subsection{Window composition}

The main theorem is stated on one certified window. If the schedule is certified on several windows, the adapted metric may change from one window to the next. The following proposition records the only cost of such a change: since each affine-tail metric controls a positive multiple of distance, switching from one metric to the next costs the inverse tail slope of the previous metric.

\begin{proposition}[Window composition]
\label{prop:window-composition}
Let
\[
  0=T_0<T_1<\cdots<T_J=T
\]
be certified windows. On window \(j\), let \(\varphi_j\) be an increasing concave cost with \(\varphi_j(0)=0\), \(\varphi_j'(0+)\le1\), affine-tail slope \(a_j>0\), and propagation rate \(\rho_j>0\). Let
\[
  h_t:=\eta_t+\zeta_t .
\]
Define recursively
\[
  \Delta_0:=\Wtwo(\widetilde\mu_{T_0},\mu_{T_0}),
\]
and, for \(j=1,\dots,J\),
\[
  \Delta_j
  :=
  e^{-\rho_j(T_j-T_{j-1})}\alpha_j\Delta_{j-1}
  +
  \int_{T_{j-1}}^{T_j}
  e^{-\rho_j(T_j-t)}h_t\,dt,
\]
where
\[
  \alpha_1:=1,
  \qquad
  \alpha_j:=a_{j-1}^{-1}\quad(j\ge2).
\]
Then
\[
  \mathsf W_{\varphi_j}(\widetilde\mu_{T_j},\mu_{T_j})
  \le
  \Delta_j
  \qquad
  \text{for every }j=1,\dots,J .
\]
Consequently, the terminal reporting bounds of \Cref{cor:terminal-w2-reporting} apply on the last window with
\[
  \Gamma_T=\Delta_J,
  \qquad
  a_R=a_J,
  \qquad
  \varphi_R=\varphi_J .
\]
\end{proposition}

\begin{proof}
For \(j=1\), the initial metric is bounded by \(\Wtwo\) because \(\varphi_1(r)\le r\). Applying \Cref{thm:main-certificate} on the first window gives the claim for \(\Delta_1\).

Assume the claim has been proved up to window \(j-1\). Since \(\varphi_j(r)\le r\) and \(\varphi_{j-1}(r)\ge a_{j-1}r\), we have
\[
  \varphi_j(r)\le a_{j-1}^{-1}\varphi_{j-1}(r).
\]
Therefore
\[
  \mathsf W_{\varphi_j}(\widetilde\mu_{T_{j-1}},\mu_{T_{j-1}})
  \le
  a_{j-1}^{-1}
  \mathsf W_{\varphi_{j-1}}(\widetilde\mu_{T_{j-1}},\mu_{T_{j-1}})
  \le
  a_{j-1}^{-1}\Delta_{j-1}.
\]
Applying \Cref{thm:main-certificate} on window \(j\) gives the displayed recursion. The terminal statement follows by applying \Cref{cor:terminal-w2-reporting} to the final adapted metric.
\end{proof}

This composition formula is deliberately metric-level. It does not require the same profile, cutoff, or rate on every window. Each window contributes its own contraction and forcing integral; switching windows only pays for the affine-tail slope needed to compare the old cost with the new one.

\section{Sharpness and separation}
\label{app:sharpness}

This appendix supports the comparison in \Cref{subsec:sharpness-main}. The main result is the fixed-height separation: scalar height, even together with eventual reserve, does not determine the radial Hardy scale. We first prove this cleanly at the profile level, where the information loss is most transparent, and then realize the same phenomenon by smooth globally Lipschitz one-dimensional drifts. Thus the profile formulation is a useful abstraction, but the phenomenon is not an artifact of abstract profiles. Finally, we record supporting sharpness statements for the exponential load factor and the terminal reporting loss.

\subsection{Main separation: profile height is not the statistic}
\label{subsec:fixed-height-separation}

The point of this result is not to construct a difficult drift field. It is to show an information loss. The one-sided height records only how deep the adverse part of a profile is; it forgets how wide the adverse barrier is. The Hardy scale depends on both.

\begin{proof}[Proof of Proposition~\ref{prop:main-height-separation}]
We build two profiles with the same adverse depth \(B\) and the same tail reserve \(m\), but with adverse regions of different width.

For the first profile, fix
\[
  \delta:=\frac{\sigma}{\sqrt m}.
\]
Let \(\kappa_1\) be a smooth profile satisfying
\[
  -B\le \kappa_1(r)\le m,
  \qquad
  \kappa_1(r)=-B\quad(0<r\le\delta),
  \qquad
  \kappa_1(r)\ge m\quad(r\ge2\delta).
\]
Then \(\kappa_1\) has one-sided height \(B\) and reserve at least \(m\) beyond \(R_1=2\delta\). Its load at \(R_1\) satisfies
\[
  A_{R_1}(\kappa_1)
  \le
  \int_0^{2\delta}Br\,dr
  =
  2B\delta^2 .
\]
By \Cref{thm:load-reserve},
\[
  \lambda_{\mathrm H}(\kappa_1,\sigma)
  \ge
  c\left(m\wedge\frac{\sigma^2}{4\delta^2}\right)
  \exp\left\{-\frac{B\delta^2}{\sigma^2}\right\}.
\]
Since \(\delta=\sigma/\sqrt m\), the right side is a strictly positive constant depending only on \(B,m,\sigma\). Set this constant to be \(c_0(B,m,\sigma)\).

For the second profile, let \(L\) be large. Choose a smooth profile \(\kappa_{2,L}\) satisfying
\[
  -B\le \kappa_{2,L}(r)\le m,
  \qquad
  \kappa_{2,L}(r)=-B\quad(0<r\le L),
  \qquad
  \kappa_{2,L}(r)\ge m\quad(r\ge2L).
\]
It again has one-sided height \(B\) and eventual reserve \(m\). On \([0,L]\),
\[
  \Phi_{\kappa_{2,L},\sigma}(r)
  =
  \exp\left\{\frac{B r^2}{4\sigma^2}\right\}.
\]
Evaluate the Hardy capacity at \(r=L/2\). For \(L\ge2\sigma/\sqrt B\),
\[
  \int_0^{L/2}\Phi_{\kappa_{2,L},\sigma}(s)^{-1}\,ds
  \ge
  c\frac{\sigma}{\sqrt B},
\]
because the integral over \(0\le s\le\sigma/\sqrt B\) is bounded below by a universal multiple of \(\sigma/\sqrt B\). Also,
\[
  \int_{L/2}^{\infty}\Phi_{\kappa_{2,L},\sigma}(s)\,ds
  \ge
  \int_{3L/4}^{L}
  \exp\left\{\frac{B s^2}{4\sigma^2}\right\}\,ds
  \ge
  \frac{L}{4}
  \exp\left\{\frac{9BL^2}{64\sigma^2}\right\}.
\]
Therefore
\[
  M_\sigma(\kappa_{2,L})
  \ge
  c\frac{\sigma L}{\sqrt B}
  \exp\left\{\frac{9BL^2}{64\sigma^2}\right\}.
\]
By \Cref{thm:hardy-completeness},
\[
  \lambda_{\mathrm H}(\kappa_{2,L},\sigma)
  \le
  \frac{2\sigma^2}{M_\sigma(\kappa_{2,L})}
  \le
  C\frac{\sigma\sqrt B}{L}
  \exp\left\{-\frac{9BL^2}{64\sigma^2}\right\}.
\]
Choosing \(L\) sufficiently large makes this upper bound at most \(\varepsilon\). Set
\[
  \kappa_2:=\kappa_{2,L},
  \qquad
  R_2:=2L .
\]
\end{proof}

\subsection{A drift-level realization}
\label{subsec:drift-level-height-separation}

The main-text separation is stated at the level of radial profiles.
We next show that the same information loss is not an artifact of abstract profiles:
it already occurs for smooth globally Lipschitz one-dimensional drifts.
The construction is simple.
The one-sided height sees the depth of an expansive plateau, while the Hardy scale also sees its width.

For a one-dimensional \(C^1\) drift \(u:\mathbb R\to\mathbb R\),
\[
  \kappa_u(r)
  =
  -\sup_{x\in\mathbb R}
  \frac{u(x+r)-u(x)}{r}
  =
  -\sup_{x\in\mathbb R}
  \frac1r\int_x^{x+r}u'(s)\,ds ,
  \qquad r>0.
\]

\begin{proposition}[Drift-level fixed-height separation]
\label{prop:drift-level-height-separation}
Fix \(\sigma>0\), \(B>0\), and \(m>0\).
For every \(\varepsilon>0\), there exist smooth globally Lipschitz one-dimensional drifts
\(u_1,u_2:\mathbb R\to\mathbb R\) such that, with \(\kappa_i:=\kappa_{u_i}\),
\[
  \sup_{r>0}[-\kappa_i(r)]_+=B,
  \qquad i=1,2,
\]
and
\[
  \kappa_i(r)\ge m
  \qquad (r\ge R_i)
\]
for some finite \(R_i\), but
\[
  \lambda_{\mathrm H}(\kappa_1,\sigma)\ge c_0(B,m,\sigma)>0,
  \qquad
  \lambda_{\mathrm H}(\kappa_2,\sigma)\le\varepsilon .
\]
Moreover, the drifts may be chosen with a common global Lipschitz bound depending only on
\(B\) and \(m\), and they may be realized as exact one-dimensional score fields with Gaussian tails.
\end{proposition}

\begin{proof}
Let
\[
  K:=2(B+m).
\]
For each length \(L>0\), choose a smooth function \(w_L:\mathbb R\to\mathbb R\) such that
\[
  -K\le w_L\le B,\qquad
  w_L=B \ \text{on } [0,L],
  \qquad
  w_L=-K \ \text{on } (-\infty,-L]\cup[2L,\infty).
\]
Define
\[
  u_L(x):=\int_0^x w_L(s)\,ds .
\]
Then \(u_L\in C^\infty(\mathbb R)\) and \(\|u_L'\|_\infty\le K\), so \(u_L\) is globally Lipschitz with a bound independent of \(L\).

We first record the radial profile facts.
Since \(u_L'\le B\) everywhere, every secant slope is at most \(B\), hence
\[
  \kappa_{u_L}(r)\ge -B
  \qquad (r>0).
\]
On the other hand, for every \(0<r\le L\), an interval of length \(r\) can be placed inside the plateau
\([0,L]\), where \(u_L'=B\). Therefore the maximal average slope at radius \(r\) is exactly \(B\), and
\[
  \kappa_{u_L}(r)=-B,
  \qquad 0<r\le L.
\]
Consequently,
\[
  \sup_{r>0}[-\kappa_{u_L}(r)]_+=B.
\]

Next we prove eventual reserve. Let \(I_L=[-L,2L]\), so \(|I_L|=3L\). For any interval
\(J\subset\mathbb R\) of length \(r\),
\[
  \int_J w_L(s)\,ds
  \le
  B\,|J\cap I_L|-K\,|J\setminus I_L|
  \le
  (B+K)3L-Kr .
\]
If \(r\ge 6L\), then
\[
  \frac1r\int_J w_L(s)\,ds
  \le
  \frac{(B+K)3L}{6L}-K
  =
  \frac{B-K}{2}
  =
  -\frac B2-m
  \le -m .
\]
Taking the supremum over \(J\) gives
\[
  \kappa_{u_L}(r)\ge m,
  \qquad r\ge 6L.
\]
Thus \(u_L\) has one-sided height \(B\) and eventual reserve \(m\), with reserve cutoff \(R_L=6L\).

We now choose two plateau widths.
For the first drift, take
\[
  L_0:=\frac{\sigma}{\sqrt m},
  \qquad
  u_1:=u_{L_0}.
\]
Since \(\kappa_{u_1}(r)\ge -B\) for all \(r\) and \(\kappa_{u_1}(r)\ge m\) for \(r\ge R_{L_0}=6L_0\),
the load at this cutoff is bounded by
\[
  A_{R_{L_0}}(\kappa_{u_1})
  \le
  \int_0^{R_{L_0}} Br\,dr
  =
  \frac{B R_{L_0}^2}{2}.
\]
Applying \Cref{thm:load-reserve} gives
\[
  \lambda_{\mathrm H}(\kappa_{u_1},\sigma)
  \ge
  c
  \left(
    m\wedge \frac{\sigma^2}{R_{L_0}^2}
  \right)
  \exp\left\{-\frac{B R_{L_0}^2}{4\sigma^2}\right\}
  =:c_0(B,m,\sigma)>0.
\]

For the second drift, take \(u_2:=u_L\) with \(L\) large.
Because \(\kappa_{u_L}(r)=-B\) for \(0<r\le L\),
\[
  \Phi_{\kappa_{u_L},\sigma}(r)
  =
  \exp\left\{\frac{Br^2}{4\sigma^2}\right\},
  \qquad 0<r\le L.
\]
Evaluate the Hardy capacity at \(r=L/2\). For \(L\ge 2\sigma/\sqrt B\),
\[
  \int_0^{L/2}\Phi_{\kappa_{u_L},\sigma}(s)^{-1}\,ds
  \ge
  c\frac{\sigma}{\sqrt B},
\]
while
\[
  \int_{L/2}^{\infty}\Phi_{\kappa_{u_L},\sigma}(t)\,dt
  \ge
  \int_{3L/4}^{L}
  \exp\left\{\frac{Bt^2}{4\sigma^2}\right\}\,dt
  \ge
  \frac L4
  \exp\left\{\frac{9BL^2}{64\sigma^2}\right\}.
\]
Hence
\[
  M_\sigma(\kappa_{u_L})
  \ge
  c\frac{\sigma L}{\sqrt B}
  \exp\left\{\frac{9BL^2}{64\sigma^2}\right\}.
\]
By \Cref{thm:hardy-completeness},
\[
  \lambda_{\mathrm H}(\kappa_{u_L},\sigma)
  \le
  \frac{2\sigma^2}{M_\sigma(\kappa_{u_L})}
  \le
  C\frac{\sigma\sqrt B}{L}
  \exp\left\{-\frac{9BL^2}{64\sigma^2}\right\}.
\]
Choosing \(L\) sufficiently large makes the last expression at most \(\varepsilon\).
Set \(u_2=u_L\). This proves the claimed separation.

Finally, since \(u_L'(x)=-K\) outside a compact interval, \(u_L(x)=-Kx+O(1)\) at both tails.
Therefore
\[
  p_L(x)
  :=
  Z_L^{-1}
  \exp\left\{\int_0^x u_L(s)\,ds\right\}
\]
is a smooth probability density with Gaussian tails, and \(u_L=\partial_x\log p_L\).
Thus the examples may also be viewed as exact one-dimensional score fields.
\end{proof}

\subsection{Structural sharpness of the loss factors}
\label{subsec:structural-sharpness-losses}

The preceding subsections explain why scalar height is not the propagation statistic. The next statements have a different role: they show that two losses in the theorem are not artifacts of the proof. The barrier result is central to the load--reserve mechanism.

\begin{proposition}[Barrier profiles force exponential capacity]
\label{prop:barrier-sharpness-app}
Fix \(R,m,\sigma>0\) and set \(L_m=\sigma/\sqrt m\). For arbitrarily large \(A\), there exist smooth profiles \(\kappa_A\) such that
\[
  A_R(\kappa_A)\asymp A,
  \qquad
  m_R(\kappa_A)\ge m,
\]
and
\[
  M_\sigma(\kappa_A)
  \ge
  c e^{A/(2\sigma^2)}R(R+L_m).
  \label{eq:barrier-capacity-lower-app}
\]
Consequently,
\[
  \lambda_{\mathrm H}(\kappa_A,\sigma)
  \le
  C\sigma^2 e^{-A/(2\sigma^2)}[R(R+L_m)]^{-1}.
\]
\end{proposition}

\begin{proof}
Choose a smooth nonnegative shell \(\psi\) supported in \([R/3,2R/3]\) and normalized by
\[
  \int_0^R\psi(r)r\,dr=1.
\]
Choose a smooth cutoff \(\chi\) with \(\chi=0\) on \([0,5R/6]\), \(\chi=1\) on \([R,\infty)\), and \(0\le\chi\le1\). Define
\[
  \kappa_A(r)=-A\psi(r)+m\chi(r).
\]
Then \(A_R(\kappa_A)=A\) and \(m_R(\kappa_A)\ge m\). Let
\[
  V_A(r)=\int_0^r\kappa_A(u)u\,du,
  \qquad
  \Phi_A(r)=e^{-V_A(r)/(2\sigma^2)}.
\]
Before the shell, \(V_A=0\). Across the shell, \(V_A\) decreases by \(A\). On a subinterval of \([2R/3,5R/6]\), \(V_A=-A\). Choose \(r_0\in(2R/3,5R/6)\) in this plateau. Since \(\Phi_A^{-1}=e^{V_A/(2\sigma^2)}\) is bounded below by a universal constant on a subinterval of \([0,R/3]\),
\[
  \int_0^{r_0}\Phi_A(s)^{-1}\,ds\ge cR.
\]
The second factor has the plateau contribution
\[
  \int_{r_0}^{5R/6}\Phi_A(t)\,dt
  \ge
  c e^{A/(2\sigma^2)}R .
\]
This gives
\[
  M_\sigma(\kappa_A)\ge c e^{A/(2\sigma^2)}R^2,
\]
which is enough when \(L_m\le R\).

It remains to obtain the \(RL_m\) contribution when \(L_m\ge R\). In this case \(mR^2\le\sigma^2\). On the transition interval \([5R/6,R]\), the positive increase of \(V_A\) is at most \(CmR^2\le C\sigma^2\). For \(t\in[R,R+c_0L_m]\), with \(c_0>0\) small,
\[
  V_A(t)-V_A(R)
  \le
  \frac{m}{2}\{(R+c_0L_m)^2-R^2\}
  \le
  Cc_0\sigma^2+C c_0^2\sigma^2.
\]
Choosing \(c_0\) sufficiently small gives
\[
  \int_R^{R+c_0L_m}\Phi_A(t)\,dt
  \ge
  c e^{A/(2\sigma^2)}L_m.
\]
Combining this with the lower bound on \(\int_0^{r_0}\Phi_A^{-1}\) gives the stated
lower bound on \(M_\sigma(\kappa_A)\). The Hardy-scale upper bound follows from \Cref{thm:hardy-completeness}.
\end{proof}

\section{Supplying and localizing radial profiles}
\label{app:auxiliary-diffusion}

This appendix supports the profile input used in the main theorem.  The main
route is analytic: Gaussian-smoothed denoising geometry gives an inverse-radius
profile when Gaussian pull dominates pairwise denoiser expansion.  For fixed
learned drifts, deterministic compact checks can supply local lower bounds, which
may then be used with an exit penalty.  These tools supply or localize the
radial profile, and are used alongside the residual envelopes
\(\eta_t,\zeta_t\) and the terminal reporting assumptions in
\Cref{thm:main-certificate,cor:terminal-w2-reporting}.

The exact infimum profile \(\kappa_u\) is only a reference object.  The theorem
uses a certified Borel minorant \(\bar\kappa_I\) satisfying
\eqref{eq:certified-profile}, and \(A_R,m_R\) are computed from this minorant.
\subsection{Analytic data-to-profile mechanisms}
\label{subsec:analytic-profile-certificates}

We first record the inverse-radius load calculation, then instantiate it by
bounded-amplitude transfer, Gaussian-smoothed denoising geometry, and
common-covariance mixtures.

\begin{lemma}[Inverse-radius profile verification]
\label{lem:inverse-radius-verification-app}
If a certified profile satisfies
\[
  \bar\kappa_I(r)\ge \alpha-\frac{\beta}{r},
  \qquad r>0,
\]
with \(\alpha>0\) and \(\beta\ge0\), then for any \(R>\beta/\alpha\),
\[
  m_R(I)\ge \alpha-\frac{\beta}{R}>0,
  \qquad
  A_R(I)\le \frac{\beta^2}{2\alpha}.
\]
In particular, if \(\beta>0\) and \(R=4\beta/\alpha\), then
\[
  m_R(I)\ge \frac{3\alpha}{4},
  \qquad
  A_R(I)\le \frac{\beta^2}{2\alpha}.
\]
If \(\beta=0\), then \(A_R(I)=0\) and \(m_R(I)\ge\alpha\) for every \(R>0\).
\end{lemma}

\begin{proof}
For \(r\ge R\), the profile lower bound gives \(\bar\kappa_I(r)\ge \alpha-\beta/R\). The adverse part can be nonzero only when \(r\le\beta/\alpha\). Hence
\[
  A_R(I)
  \le
  \int_0^{\beta/\alpha}\left(\frac{\beta}{r}-\alpha\right)r\,dr
  =
  \frac{\beta^2}{2\alpha}.
\]
The remaining claims are immediate.
\end{proof}

\begin{proposition}[Bounded-amplitude transfer]
\label{prop:amplitude-perturbation-app}
\label{prop:amplitude-certificate-app}
Let \(u_0:\R^d\to\R^d\) have radial profile at least \(\alpha>0\), meaning
\[
  -\frac{\langle u_0(x)-u_0(y),x-y\rangle}{\|x-y\|^2}
  \ge
  \alpha
  \qquad (x\ne y).
\]
Let \(\widehat u=u_0+\delta\) with \(\|\delta\|_\infty\le\varepsilon\). Then
\[
  \kappa_{\widehat u}(r)\ge \alpha-\frac{2\varepsilon}{r},
  \qquad r>0.
\]
Consequently, for \(R=4\varepsilon/\alpha\) when \(\varepsilon>0\),
\[
  m_R\ge \alpha/2,
  \qquad
  A_R\le \frac{2\varepsilon^2}{\alpha}.
\]
If \(\varepsilon=0\), then \(A_R=0\) and \(m_R\ge\alpha\) for every \(R>0\).
\end{proposition}

\begin{proof}
For \(x\ne y\), set \(r=\|x-y\|\). Since \(\widehat u=u_0+\delta\),
\[
-\frac{\langle \widehat u(x)-\widehat u(y),x-y\rangle}{r^2}
\ge
\alpha-\frac{\|\delta(x)\|+\|\delta(y)\|}{r}
\ge
\alpha-\frac{2\varepsilon}{r}.
\]
Taking the infimum over all pairs at distance \(r\) gives the profile bound. The load--reserve estimates follow from \Cref{lem:inverse-radius-verification-app} with \(\beta=2\varepsilon\).
\end{proof}

\begin{proposition}[Denoising-expansion profile template]
\label{prop:denoising-expansion-profile-app}
Fix a window \(I\) and write \(s=T-t\). Suppose a Gaussian-smoothed law admits
\[
  \nabla\log p_s(x)=-\tau_s\{x-m_s(x)\},
  \qquad
  m_s(x)=\E[C_s\mid X_s=x],
\]
where \(C_s\) is the center of the isotropic Gaussian smoothing kernel.  In the
additive VE normalization \(C_s=X_0\), while in an affine VP normalization one may
take \(C_s=a_sX_0\).  Assume that, for every \(x,y\),
\[
  \langle m_s(x)-m_s(y),x-y\rangle
  \le
  \ell_s\|x-y\|^2+D_s\|x-y\|.
\]
Assume also that
\[
  \|s_\theta(\cdot,s)-\nabla\log p_s\|_\infty\le E_s.
\]
For
\[
  \widehat b_t(x)=f(s)x+g(s)^2s_\theta(x,s),
\]
define
\[
  \alpha_I:=\essinf_{s\in T-I}\{g(s)^2\tau_s(1-\ell_s)-f(s)\},
  \qquad
  \beta_I:=\esssup_{s\in T-I}g(s)^2(\tau_sD_s+2E_s).
\]
If \(\alpha_I>0\), then
\[
  \bar\kappa_I(r)\ge\alpha_I-\frac{\beta_I}{r},
  \qquad r>0.
\]
Consequently, if \(\beta_I>0\) and \(R=4\beta_I/\alpha_I\), then
\[
  A_R(I)\le \frac{\beta_I^2}{2\alpha_I},
  \qquad
  m_R(I)\ge \frac{3\alpha_I}{4}.
\]
If \(\beta_I=0\), then \(A_R(I)=0\) and \(m_R(I)\ge\alpha_I\) for every \(R>0\).
\end{proposition}

\begin{proof}
Set \(r=\|x-y\|\). The exact score identity gives
\[
-\frac{\langle \nabla\log p_s(x)-\nabla\log p_s(y),x-y\rangle}{r^2}
=
\tau_s\left(1-
\frac{\langle m_s(x)-m_s(y),x-y\rangle}{r^2}\right).
\]
The denoising-expansion assumption therefore implies
\[
-\frac{\langle \nabla\log p_s(x)-\nabla\log p_s(y),x-y\rangle}{r^2}
\ge
\tau_s(1-\ell_s)-\frac{\tau_sD_s}{r}.
\]
The learned-score error contributes at worst \(-2E_s/r\) to the radial profile.
Multiplication by \(g(s)^2\) and addition of the linear term \(f(s)x\), which
contributes \(-f(s)\), gives
\[
  \kappa_{\widehat b_t}(r)
  \ge
  g(s)^2\tau_s(1-\ell_s)-f(s)
  -
  \frac{g(s)^2(\tau_sD_s+2E_s)}{r}.
\]
Taking the essential infimum over the window gives the displayed certified
minorant, and the load--reserve estimates follow from
\Cref{lem:inverse-radius-verification-app}.
\end{proof}

\paragraph{Common-covariance Gaussian mixtures.}
The mixture certificate is a special case of \Cref{prop:denoising-expansion-profile-app}. If
\[
  p_s=\sum_{k=1}^K w_k\mathcal N(m_k(s),\tau_s^{-1}I_d)
\]
has common covariance and component means of diameter \(D_s\), then
\[
  \nabla\log p_s(x)=-\tau_s\{x-\bar m_s(x)\},
\]
where \(\bar m_s(x)\) is the posterior mean of the component centers. Since
\(\bar m_s(x)\) lies in the convex hull of the centers,
\[
  \|\bar m_s(x)-\bar m_s(y)\|\le D_s,
\]
so the denoising template applies with \(\ell_s=0\). Thus
\[
  \alpha_I=\essinf_{s\in T-I}\{g(s)^2\tau_s-f(s)\},
  \qquad
  \beta_I=\esssup_{s\in T-I}g(s)^2(\tau_sD_s+2E_s).
\]
This recovers the common-covariance Gaussian-mixture case described in \Cref{subsec:data-profile}.

\paragraph{Schedule calibration for the mixture certificate.}
The condition \(\alpha_I>0\) is a window-level dominance condition after writing
\(\widehat b_t(x)=f(s)x+g(s)^2s_\theta(x,s)\), \(s=T-t\).  For VE smoothing,
\[
  X_s=X_0+\sqrt{\upsilon(s)}Z,\qquad
  \upsilon(s)=\int_0^s g(u)^2du,
\]
an initial common-covariance mixture with component variance \(v_0\) and center
diameter \(D_0\) has
\[
  \tau_s=(v_0+\upsilon(s))^{-1},\quad D_s=D_0,
  \quad f(s)=0 .
\]
For VP smoothing,
\[
  v_s=\bar a_sv_0+(1-\bar a_s),\qquad
  \tau_s=v_s^{-1},\qquad D_s=\sqrt{\bar a_s}D_0,
\]
and in sampler time \(f(s)=\lambda(s)/2\), \(g(s)^2=\lambda(s)\).  Substituting
these quantities into \(\alpha_I,\beta_I\) gives the certified window constants.
If \(\alpha_I\le0\) on a subwindow, the inverse-radius template simply makes no
claim there.

\paragraph{Optional posterior-covariance refinement.}
For mixtures, the crude diameter bound can be sharpened through
\[
  D\bar m_s(x)=\tau_s\operatorname{Cov}(m\mid x),
  \qquad
  \|D\bar m_s(x)\|_{\rm op}\le \tau_sD_s^2/4 .
\]
Thus
\(\|\bar m_s(x)-\bar m_s(y)\|\le
\min\{D_s,\tau_sD_s^2\|x-y\|/4\}\).  Substituting this into the denoising
expansion bound can remove the inverse-radius singularity near the origin; with
exact scores this yields a zero-load certificate on windows where
\(g(s)^2\tau_s(1-\tau_sD_s^2/4)>f(s)\).

\subsection{Compact verification and diagnostics}
\label{subsec:fixed-drift-certification}

The analytic mechanisms above are structural.  For a fixed learned drift on a
compact region \(K\), one can also verify conservative local lower bounds on the
radial projection.  This is a profile-verification device for a fixed learned drift, complementary
to statistical training analyses.

The finite-cover certificate is stated away from the diagonal because
\[
  q(t,x,e,r)
  :=
  -\frac{\langle \widehat b_t(x+re)-\widehat b_t(x),e\rangle}{r}
\]
contains a factor \(1/r\). We cover radii \(r\in[r_{\min},D_K]\), where \(D_K\) is the diameter of \(K\), and handle \(0<r<r_{\min}\) separately by a local Jacobian or verified directional-difference lower bound.

Fix a compact set \(K\subset\R^d\), a time window \(I\), and bins \(J_\ell\) covering \([r_{\min},D_K]\). For \(t\in I\), \(x,x+re\in K\), \(e\in\mathbb S^{d-1}\), and \(r\in J_\ell\), set
\[
  \mathcal Z_\ell
  :=
  \{(t,x,e,r):t\in I,\ x,x+re\in K,\ e\in\mathbb S^{d-1},\ r\in J_\ell\}.
\]
For the small-radius region define the line-segment hull
\[
  K_\star(r_{\min})
  :=
  \{x+\theta(y-x):x,y\in K,\ \|x-y\|\le r_{\min},\ \theta\in[0,1]\}.
\]
A certified number \(\kappa_0\) is a near-diagonal certificate if
\begin{equation}
  -\langle D\widehat b_t(z)e,e\rangle\ge \kappa_0
  \qquad
  (t\in I,\ z\in K_\star(r_{\min}),\ e\in\mathbb S^{d-1}).
  \label{eq:small-radius-jacobian-certificate}
\end{equation}
For nonsmooth networks, \eqref{eq:small-radius-jacobian-certificate} can be replaced by any verified directional difference bound on \(0<r\le r_{\min}\).

\begin{lemma}[Deterministic finite-cover certificate]
\label{lem:finite-cover-profile-certificate}
Assume that \(q\) is \(L_\ell\)-Lipschitz on \(\mathcal Z_\ell\), with respect to a chosen product metric, and let \(\mathcal G_\ell\) be an \(\varepsilon_\ell\)-net of \(\mathcal Z_\ell\). If a certified evaluator proves
\[
  q(z)\ge \kappa_\ell+L_\ell\varepsilon_\ell
  \qquad\text{for every }z\in\mathcal G_\ell,
\]
then for every \(t\in I\) and every pair \(x,y\in K\) with \(\|x-y\|\in J_\ell\),
\[
  -\frac{\langle \widehat b_t(x)-\widehat b_t(y),x-y\rangle}{\|x-y\|^2}
  \ge
  \kappa_\ell .
\]
If, in addition, \(\kappa_0\) satisfies \eqref{eq:small-radius-jacobian-certificate}, then
\[
  \bar\kappa_K(r)
  :=
  \begin{cases}
    \kappa_0, & 0<r<r_{\min},\\
    \min_{\ell:r\in J_\ell}\kappa_\ell, & r\in[r_{\min},D_K]
  \end{cases}
\]
is a conservative lower radial profile on \(K\).
\end{lemma}

\begin{proof}
For any \((t,x,e,r)\in\mathcal Z_\ell\), choose \(z'\in\mathcal G_\ell\) within distance \(\varepsilon_\ell\). The Lipschitz bound gives \(q(t,x,e,r)\ge q(z')-L_\ell\varepsilon_\ell\ge\kappa_\ell\). For a pair \(x,y\in K\), take \(r=\|x-y\|\) and \(e=(y-x)/r\). Then
\[
q(t,x,e,r)
=
-\frac{\langle \widehat b_t(y)-\widehat b_t(x),y-x\rangle}{r^2}
=
-\frac{\langle \widehat b_t(x)-\widehat b_t(y),x-y\rangle}{r^2}.
\]
This proves the binwise pairwise lower bound.

For \(0<r\le r_{\min}\), write \(y=x+re\). By the fundamental theorem of calculus and the definition of \(K_\star(r_{\min})\),
\[
-\frac{\langle \widehat b_t(y)-\widehat b_t(x),y-x\rangle}{r^2}
=
\int_0^1
-\langle D\widehat b_t(x+\theta re)e,e\rangle\,d\theta
\ge \kappa_0 .
\]
This proves the near-diagonal part and hence the conservative profile.
\end{proof}

\paragraph{Randomized covers.}
One can replace a deterministic net by randomized certification queries under
additional small-ball coverage and certified-evaluator assumptions. Such a
statement would be a randomized verification result for a fixed drift on a fixed
compact region. We keep the deterministic finite-cover certificate as the main
verification template.

\paragraph{Diagnostic binning.}
\label{subsec:empirical-profile-binning}
Empirical radial bins can estimate the observed load--reserve profile along
sampler trajectories.  For pairs with
\(r_{ij,t}=\|x_t^{(i)}-x_t^{(j)}\|>0\), compute
\[
  q_{ij,t}:=-\frac{\langle \widehat b_t(x_t^{(i)})-\widehat b_t(x_t^{(j)}),
  x_t^{(i)}-x_t^{(j)}\rangle}{\|x_t^{(i)}-x_t^{(j)}\|^2} .
\]
Binning these values by radius gives empirical summaries and diagnostic
estimates of \(A_R\) and \(m_R\).  They become certificates when the binwise
lower bounds include conservative sampling, approximation, and evaluator margins
covering all relevant pairs; \Cref{lem:finite-cover-profile-certificate} records
one deterministic way to add such margins on a fixed compact set.

\subsection{Localized use with an exit penalty}
\label{subsec:localized-profile-certificate}

A compact certificate applies on the region \(K\). We
use it to construct a localized learned reference flow, then insert
score and solver residuals by the same perturbation lemma as in the main proof.
The extension of the local profile outside \(K\) is an analytic device for
building \(\varphi_R\); localization is accounted for by a terminal exit penalty.

Let \(K\subset\R^d\) have diameter \(D_K\), and suppose the learned drift has a
certified lower profile \(\bar\kappa_K\) for all pairs in \(K\).  For
\(0<R<D_K\), set
\[
  A_R^K=\int_0^R[-\bar\kappa_K(r)]_+r\,dr,
  \qquad
  m_R^K=\operatorname*{ess\,inf}_{r\in[R,D_K]}\bar\kappa_K(r).
\]
If \(A_R^K<\infty\) and \(m_R^K>0\), extend \(\bar\kappa_K\) beyond \(D_K\) by
\(m_R^K\) and construct \((\varphi_R,\rho_R)\) from
\Cref{thm:affine-tail-metric}.  Same-drift reflection supplies the contraction
estimate below before exit; if the reflected learned pair exits \(K\) before
\(T\) with probability at most \(\delta_K\) and has terminal second-moment scale
\(M_T\), the lost part is bounded by \(M_T\delta_K^{1/2}\) because
\(\varphi_R(r)\le r\).

\begin{proposition}[Localized profile propagation]
\label{prop:localized-profile-propagation}
Assume a localized learned reference evolution \(\widehat P^K_{s,t}\) satisfies
\[
  \mathsf W_{\varphi_R}(\nu\widehat P^K_{s,t},\xi\widehat P^K_{s,t})
  \le
  e^{-\rho_R(t-s)}\mathsf W_{\varphi_R}(\nu,\xi)
  \qquad(0\le s\le t\le T)
\]
for the laws under consideration.  Let \(\widetilde\mu_t^K\) and \(\mu_t^K\) be
localized implemented and ideal laws which are perturbations of this reference
flow, with envelopes
\[
  \int\|v_t^K(x)\|\,\widetilde\mu_t^K(dx)\le \zeta_t,
  \qquad
  \int\|w_t^K(x)\|\,\mu_t^K(dx)\le \eta_t .
\]
Assume their terminal discrepancy from the actual laws is bounded by
\begin{equation}
  \mathsf W_{\varphi_R}(\widetilde\mu_T,\widetilde\mu_T^K)
  +
  \mathsf W_{\varphi_R}(\mu_T^K,\mu_T)
  \le
  M_T\delta_K^{1/2} .
  \label{eq:localized-terminal-discrepancy}
\end{equation}
Then
\[
\begin{aligned}
  \mathsf W_{\varphi_R}(\widetilde\mu_T,\mu_T)
  \le\;&
  e^{-\rho_RT}\mathsf W_{\varphi_R}(\widetilde\mu_0^K,\mu_0^K)
  +
  \int_0^T e^{-\rho_R(T-t)}(\eta_t+\zeta_t)\,dt  \\
  &+
  M_T\delta_K^{1/2}.
\end{aligned}
\]
If the localized and actual processes start from the same initial laws, then
\(\widetilde\mu_0^K=\widetilde\mu_0\) and \(\mu_0^K=\mu_0\).
\end{proposition}

\begin{proof}
Apply \Cref{lem:contracted-flow-perturbation-app} to the contracted reference
\(\widehat P^K_{s,t}\), with \(\alpha_t=\widetilde\mu_t^K\),
\(\beta_t=\mu_t^K\), and perturbations \(v_t^K,w_t^K\).  This gives the adapted
cost bound for \((\widetilde\mu_T^K,\mu_T^K)\).  The triangle inequality for
\(\mathsf W_{\varphi_R}\), together with
\eqref{eq:localized-terminal-discrepancy}, gives the displayed estimate for the
actual terminal laws.  No non-coalescing reflection coupling between ideal and
implemented processes is used.
\end{proof}

Applying \Cref{cor:terminal-w2-reporting} to the same affine-tail cost gives the
localized terminal \(\Wtwo^2\) reports with \(\Gamma_T\) replaced by the right-hand
side of \Cref{prop:localized-profile-propagation}.  Standard Lyapunov,
path-moment, or stopped-process estimates may be used to supply the exit
probability \(\delta_K\) and moment scale \(M_T\).

\subsection{Residual control and profile certification}
\label{subsec:forcing-does-not-certify}

The propagation theorem keeps residual envelopes and radial profiles as separate
inputs. The following construction explains the reason: averaged score residual
control alone may coexist with an arbitrarily adverse worst-case radial profile.

\begin{proposition}[Average forcing alone need not certify radial geometry]
\label{prop:forcing-does-not-certify-geometry}
Let \(P\) be any probability law on \(\R^d\), and let \(s_0:\R^d\to\R^d\) be a \(C^1\) vector field. For every \(\varepsilon>0\) and \(M>0\), there exists a smooth compactly supported perturbation \(\delta\in C_c^\infty(\R^d;\R^d)\) such that
\[
  \|\delta\|_\infty\le\varepsilon,
  \qquad
  \left(\int \|\delta(x)\|^2\,P(dx)\right)^{1/2}\le\varepsilon,
\]
but
\[
  \sup_{r>0}[-\kappa_{s_0+\delta}(r)]_+\ge M .
\]
Moreover, \(\delta\) may be chosen as the gradient of a smooth compactly supported potential.
\end{proposition}

\begin{proof}
Fix \(z\in\R^d\) and \(e\in\mathbb S^{d-1}\). Let
\[
  a:=\langle Ds_0(z)e,e\rangle,
  \qquad
  K:=M+1+|a|.
\]
Choose \(\chi\in C_c^\infty(B(0,1))\) with \(\chi\equiv1\) on \(B(0,1/2)\). For \(h>0\), define
\[
  \Psi_h(x):=\frac K2\langle x-z,e\rangle^2\chi\!\left(\frac{x-z}{h}\right),
  \qquad
  \delta_h:=\nabla\Psi_h .
\]
Then \(\delta_h\in C_c^\infty(\R^d;\R^d)\), and \(\|\delta_h\|_\infty\le C_\chi Kh\) for a constant \(C_\chi\) depending only on \(\chi\). Choose \(h\) so small that \(C_\chi Kh\le\varepsilon\). Since \(P\) is a probability law, the displayed \(L^2(P)\) bound also holds.

Near \(z\), the cutoff equals one. Hence, for all sufficiently small \(t>0\),
\[
  \delta_h(z+te)=Kte,
  \qquad
  \delta_h(z)=0 .
\]
For \(u_h=s_0+\delta_h\),
\[
  \frac{\langle u_h(z+te)-u_h(z),te\rangle}{t^2}
  =
  \frac{\langle s_0(z+te)-s_0(z),e\rangle}{t}
  +K
  \longrightarrow
  a+K
  \ge M+1 .
\]
Thus \(\kappa_{u_h}(t)\le -M\) for all sufficiently small \(t\), and \(\sup_{r>0}[-\kappa_{u_h}(r)]_+\ge M\). Taking \(\delta=\delta_h\) proves the claim. The construction used \(\delta_h=\nabla\Psi_h\), so the perturbation is conservative.
\end{proof}

\paragraph{Interpretation.}
The proposition is a worst-case certification statement. Score losses control
injected residuals under the noised law, while radial profiles certify all-pairs
learned-flow stability. Solver residuals are modular in the same way: local
geometry may be used to derive \(\zeta_t\), and once such an envelope is
available it enters the perturbation bound additively.

\section{Additional Related Work}
\label{app:additional-related-work}

\paragraph{Diffusion-error decompositions and complementary inputs.}
Many nonasymptotic diffusion analyses decompose sampling error into score
estimation, initialization, propagation, discretization, and terminal conversion
terms, often closing propagation through explicit \(W_2\)-type multipliers. Our
decomposition is modular in the same sense, and uses a profile-adapted stability
input.
Score-estimation, KL/Girsanov, weak-regularity, stochastic-localization, and
Gaussian-tail results may provide forcing envelopes, terminal moment or tail
bounds, or analytic lower bounds for structured exact-score profiles
\citep{benton2024nearly,wang2024gaussian_tail,conforti2024weak,
conforti2023kl,gentilonisilveri2025beyond,kremling2025pfodeweaklog,
wibisono2024optimal_score,cole2024scorebased,li2025odt,
li2024provable_acceleration,liang2025lowdim}. Solver analyses may provide
\(\zeta_t\), while tail or functional-inequality estimates may provide terminal
reporting inputs. These inputs are complementary to our interface: once supplied,
the certified learned-drift radial profile determines how they are propagated in
an adapted Wasserstein cost.

\paragraph{Geometry-adapted contraction.}
Our reflection-coupling step and concave radial costs build on the Eberle-type
contraction literature
\citep{lindvall1986reflection,chen1989coupling,eberle2016reflection,
eberle2019couplings,majka2020nonasymptotic}. Classical radial-coupling results
prove contraction of a known Markov semigroup in an adapted metric. In our sampler setting, the learned drift supplies the pairwise radial profile used to
construct the metric, while ideal--learned score mismatch and solver residuals
enter later as residual inputs. In the fixed-semigroup regime,
our propagation inequality reduces to contraction in the constructed concave
radial cost; when \(\kappa(r)\ge m\), the zero-load case recovers
\(\rho_R \gtrsim m\wedge \sigma^2/R^2\) up to constants. For profiles with an
adverse core, the exponential load factor is the slope loss paid by the concave
metric before reaching tail reserve. Matrix-metric contraction
\citep{monmarche2023almost} provides a natural route beyond scalar-isotropic
noise, but would call for certifying a metric-dependent geometry.

\paragraph{Solvers and sampler interfaces.}
Work on PFODEs, DDIM-type samplers, discretization, and high-accuracy schemes
\citep{chen2023pfode_fast,gao2024pfode,li2025odt,liang2025lowdim,
yu2025discretization,chen2026high_accuracy,pfarr2026higher_order}
is also complementary. The theorem treats \(\zeta_t\) as a law-level drift
perturbation of an adapted continuous interpolation; the envelope itself may be derived using geometry. A sharper solver analysis may build local
Jacobian, Hessian, commutator, profile, or step-size information into
\(\zeta_t\) before this residual is propagated additively by the contracted
learned flow. Deterministic or noise-mismatched samplers require a separate
comparison or extension before the scalar reverse-SDE propagation interface
applies.

\paragraph{Certification and localized use.}
Deterministic compact verification is a profile-verification device for fixed
learned drifts and complements statistical training analyses. It is complementary to Harris-type forgetting arguments
\citep{strasman2026forgetting}, deterministic or probability-flow sampler
analyses \citep{li2024unified_deterministic,beyler2025deterministic_stochastic},
semiconvex Wasserstein analyses \citep{bruno2025semiconvex}, and statistical
score-estimation results, including heavy-tail and intrinsic-dimension
guarantees \citep{yu2026heavy_tailed,chakraborty2026generalization}. Such
results can provide forcing, containment, or terminal inputs. The compact route
supplies the geometric input when sampler trajectories remain in a low-complexity
region, or when normal directions are controlled by additional analytic,
architectural, interval, or verified-network bounds.

\section{Limitation and broader impacts}
\label{app:limitations-checklist}

\paragraph{Limitation.}
The theorem is a certified-window propagation result. Guidance terms that preserve scalar isotropic noise can be absorbed into the
learned drift and are covered once the guided drift satisfies a certified radial
profile. Anisotropic, degenerate, or state-dependent diffusion matrices would
instead call for a matrix- or Riemannian-adapted contraction metric. The certified radial profile is a pairwise learned-flow
input, while score and solver errors are law-level residual inputs. Deterministic
compact checks, verified-network bounds, and empirical binning provide
verification or localization mechanisms for fixed learned drifts.

Terminal \(\Wtwo^2\) reporting uses additional tail, moment, or support
information for the ideal and implemented laws. The radial profile controls
propagation in the adapted affine-tail cost, while the terminal assumptions
supply quadratic reporting. Within these inputs, the load dependence is
structural by the barrier examples, though constants and certification methods
may improve under additional model structure. A natural next step is to develop matrix- or Riemannian-adapted load--reserve certificates for reverse samplers whose diffusion geometry is not scalar-isotropic, while preserving the same separation between profile construction, residual propagation, and terminal reporting.

\paragraph{Broader impacts.}
This is foundational mathematical work and does not release a model, dataset, or
deployment system. Its positive impact is to clarify which geometric object must
be certified for profile-adapted Wasserstein propagation with terminal
\(\Wtwo^2\) reporting, which may eventually support more reliable verification
and diagnosis of diffusion samplers. Possible negative impacts are indirect:
stronger guarantees for generative models could contribute to more capable
synthetic-media systems, which require application-specific safeguards such as
provenance, access control, and misuse mitigation. The paper itself does not
lower deployment barriers for such systems.


\end{document}